\documentclass[lettersize,journal]{IEEEtran}
\usepackage{amsmath,amsfonts}
\usepackage{algorithmic}
\usepackage{algorithm}
\usepackage{array}
\usepackage[caption=false,font=normalsize,labelfont=sf,textfont=sf]{subfig}
\usepackage{textcomp}
\usepackage{stfloats}
\usepackage{url}
\usepackage[table]{xcolor}
\usepackage{verbatim}
\usepackage{graphicx}
\usepackage{booktabs}       
\usepackage{amsfonts}       
\usepackage{nicefrac}       
\usepackage{microtype}      
\usepackage{xcolor}         
\usepackage{microtype}
\usepackage{graphicx}
\usepackage{algorithm}
\usepackage{algorithmic}
\usepackage{wrapfig}
\usepackage{booktabs} 
\usepackage{hyperref}
\usepackage{booktabs}    
\usepackage{amsthm}
\newtheorem{theorem}{Theorem}

\usepackage{pifont}
\usepackage{arydshln}

\usepackage{setspace}  
\usepackage{multirow}
\usepackage{array}   
\usepackage{amsmath}
\usepackage{amssymb}
\usepackage{mathtools}
\usepackage{amsthm}
\usepackage{soul}
\usepackage{amsmath}
\usepackage{amssymb}
\usepackage{graphicx}
\usepackage{float} 
\usepackage{geometry}
\usepackage{url}
\usepackage{booktabs} 
\usepackage[capitalize,noabbrev]{cleveref}
\usepackage{graphicx}   
\usepackage{hyperref}  
\usepackage{orcidlink}
\usepackage{cite}
\begin{document}

\title{STAGE: Segmentation-oriented Industrial Anomaly Synthesis via Graded Diffusion with Explicit Mask Alignment}

\author{Xichen Xu$^{1}$\orcidlink{0009-0006-1442-7189}, 
        Yanshu Wang$^{1}$\orcidlink{0009-0003-0202-2289},~\IEEEmembership{Student Member,~IEEE}, 
        Jinbao Wang\orcidlink{0000-0001-5916-8965},~\IEEEmembership{Member,~IEEE}, 
        Qunyi Zhang\orcidlink{0009-0004-7452-1500},  
        Xiaoning Lei\orcidlink{0009-0002-0518-7903},     
        Guoyang Xie$^{*}$\orcidlink{0000-0001-8433-8153},  
        Guannan Jiang$^{*}$\orcidlink{0000-0003-4355-5711},     
        Zhichao Lu\orcidlink{0000-0002-4618-3573},~\IEEEmembership{Member,~IEEE}, 
\thanks{ Xichen Xu, Yanshu Wang and Qunyi Zhang are with the Global Institute of Future Technology, Shanghai Jiao Tong University, Shanghai, China (e-mail:neptune\_2333@sjtu.edu.cn, isaac\_wang@sjtu.edu.cn, zqyeleven@sjtu.edu.cn)

Jinbao Wang is with the School of Artificial Intelligence, Shenzhen University, Shenzhen, China (e-mail: wangjb@szu.edu.cn).

Xiaoning Lei and Guannan Jiang are with the Department of Intelligent Manufacturing, CATL, Ningde, China (e-mail: leixn01@catl.com, jianggn@catl.com).

Guoyang Xie is with the Department of Intelligent Manufacturing, CATL,Ningde, China, and also with the Department of Computer Science, City University of Hong Kong, Hong Kong, China (e-mail: guoyang.xie@ieee.org).

Zhichao Lu is with the Department of Computer Science, City University of Hong Kong, Hong Kong, China (e-mail: luzhichaocn@gmail.com).

Code is available at \url{https://github.com/Chhro123/STAGE}.

$^{1}$ and $^{*}$ indicate contributed equally and co-corresponding authors, respectively. 
}
}
%



\maketitle

\begin{abstract}
{S}egmentation-oriented {I}ndustrial {A}nomaly {S}ynthesis (SIAS) plays a pivotal role in enhancing the performance of downstream anomaly segmentation, as it offers an effective means of expanding abnormal data. However, existing SIAS methods face several critical limitations: (\romannumeral1) \textbf{the synthesized anomalies often lack intricate texture details and fail to align precisely with the surrounding background}, and (\romannumeral2) \textbf{they struggle to generate fine-grained, pixel-level anomalies}. To address these challenges, we propose {S}egmen{T}ation-oriented {A}nomaly synthesis via {G}raded diffusion with {E}xplicit mask alignment, termed \textbf{STAGE}. STAGE introduces a novel \emph{anomaly inference} strategy that incorporates clean background information as a prior to guide the denoising distribution, enabling the model to distinguish more effectively and highlight abnormal foregrounds. Furthermore, it employs a \emph{graded diffusion} framework that employs an anomaly-only branch to explicitly record local anomalies during both the forward and reverse processes, ensuring that tiny anomalies are not overlooked. Finally, STAGE incorporates the \emph{explicit mask alignment} (EMA) strategy to progressively align the synthesized anomalies with the background, resulting in context-consistent and structurally coherent generations. Extensive experiments on the MVTec and BTAD datasets demonstrate that STAGE achieves state-of-the-art performance in SIAS, which in turn improves downstream anomaly segmentation performance.
\end{abstract}


\begin{figure}[ht]
  \centering
  \includegraphics[width=\linewidth]{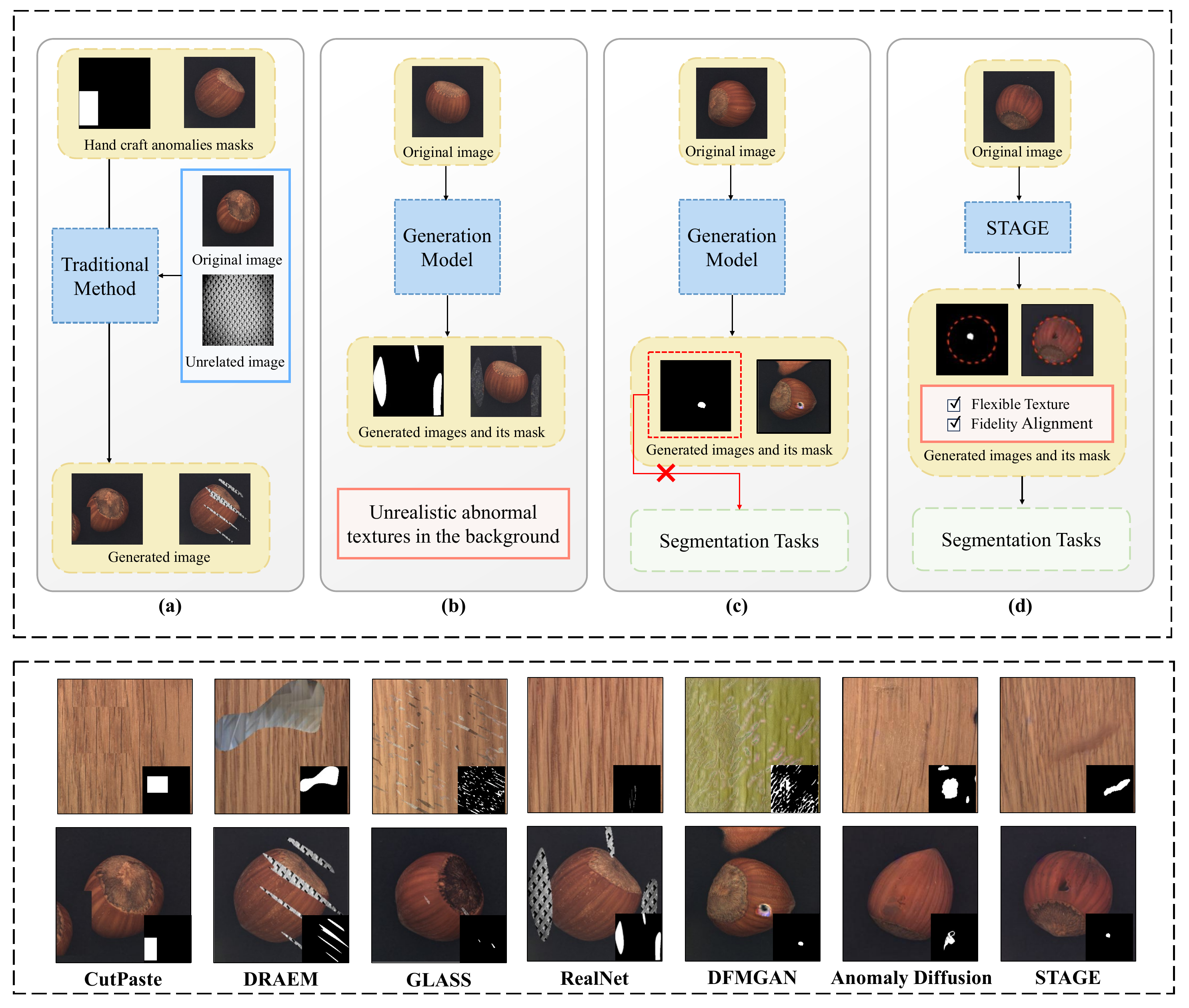}
  \vspace{-20pt}
    \caption{
        \textit{\textbf{Top:}} Traditional and generative models still compromise between three challenges—(a) Limited Texture Fidelity; (b) Misalignment between the abnormal region and background; (c) Ignoring pixel-level anomaly segmentation. \textbf{(d) STAGE overcomes them simultaneously.}
        \textit{\textbf{Bottom:}} Generated abnormal samples from different models trained on {wood} and {hazel\_nut}. STAGE addresses these issues best.
    }
  \label{intro}
\end{figure}
\section{Introduction}\label{submission}

In modern industrial manufacturing, anomaly segmentation has emerged as an essential vision task for ensuring product quality and promptly detecting latent anomalies. Its goal is precise pixel-level localization of anomalies. However, the development of high-performance anomaly segmentation models faces a fundamental bottleneck. Training such models necessitates a substantial quantity of abnormal samples and meticulously annotated, high-quality pixel-level annotations. Moreover, in real-world industrial production scenarios, the high production yield ratio implies that obtaining a sufficient number of naturally occurring abnormal samples is extremely challenging~\cite{Baitieva_2024_CVPR,pmlr-v202-chu23b,qiu2022latent}. Furthermore, the cost associated with pixel-level annotation, which often requires painstaking human labor, is prohibitively expensive. To address this data barrier, researchers have explored a diverse array of techniques. Training-free methods~\cite{li2021cutpaste,zavrtanik2021draem} and deep generative approaches~\cite{dfmgan,hu2024anomalydiffusion} have been actively used today to generate synthetic abnormal samples. Despite these concerted efforts, existing segmentation-oriented industrial anomaly synthesis (SIAS) methods still fall short of the demands of segmentation-oriented applications.

As illustrated in Fig.~\ref{intro}, current SIAS methods fail to simultaneously address the three key issues that are essential for industrial anomaly segmentation. These issues include:
(\romannumeral1) \textbf{Limited Texture Fidelity.} Hand-crafted SIAS methods such as Cutpaste~\cite{li2021cutpaste} and DRAEM~\cite{zavrtanik2021draem} are training-free and simple to apply, but lack the representational capacity to synthesize realistic textures, as shown in Fig.~\ref{intro}(a).
(\romannumeral2) \textbf{Misalignment between Abnormal Region and Background.} Deep generative methods~\cite{niu2020defect,liu2019multistage} are known to synthesize more realistic anomalies, but they often struggle to achieve precise pixel-level alignment between the abnormal regions and the background. For instance, DFMGAN~\cite{dfmgan} frequently displaces anomalies from their intended regions, as shown in Fig.~\ref{intro}(b).
(\romannumeral3) \textbf{Lack of Segmentation-oriented Optimization.} Numerous methods focus on classification and localization, often overlooking the critical role of anomaly synthesis in anomaly segmentation. Although AnomalyDiffusion~\cite{hu2024anomalydiffusion} leverages textual and positional cues to control anomalies, its optimization goal is mainly oriented towards detection, yielding suboptimal pixel-level results, as shown in Fig.~\ref{intro}(c).

To address these challenges, we propose \textbf{STAGE}: {S}egmen{T}ation-oriented {A}nomaly synthesis via {G}raded diffusion with {E}xplicit mask alignment, specifically designed to overcome the limitations of existing SIAS methods. Our method is characterized by three key components: \textbf{Anomaly Inference (AIF)}, \textbf{Graded Diffusion (GD)}, and \textbf{Explicit Mask Alignment (EMA)}. (\romannumeral1) To emphasize anomalous regions while preserving known background structure, AIF imposes distribution-level constraints and performs conditional sampling directly on the posterior, ensuring strict mask-wise disentanglement. Unlike common methods such as RePaint~\cite{lugmayr2022repaint}, which integrates background and generated content throughout denoising without preventing background leakage, AIF leverages the known background within the denoising distribution and infers anomalies only inside the mask. (\romannumeral2) Since tiny anomalies are easily suppressed by the dominant background during generation, GD adopts a dual-branch denoising strategy: the anomaly-only branch is dedicated to modeling abnormal regions and is selectively activated within a designated interval, while the anomaly-aware branch governs the remaining steps. This alternating scheme preserves small or low-contrast anomalies throughout the diffusion trajectory. (\romannumeral3) Because SIAS often relies on global background information to generate visually coherent results~\cite{jin2024dualanodiff,yang2025defect}, EMA is designed to synthesize accurate anomaly content under dynamic background guidance. It gradually adjusts the mask via a linear schedule, enabling a smooth transition from global context modeling to precise, mask-aligned anomaly synthesis. Together, these components enable STAGE to achieve high-quality, segmentation-oriented anomaly synthesis beyond simple copy-paste strategies.

The main contributions of this paper are summarized as follows:
\begin{itemize}

\item To boost the model's ability to synthesize anomaly textures, we propose an anomaly inference mechanism. This background-guided formulation extends the diffusion sampling process with clean background information, suppresses redundant background reconstruction, and enhances the model's focus on anomaly synthesis within masked regions.

\item To prevent tiny anomalies from being overlooked, we propose a dual-branch graded diffusion strategy. During training, an auxiliary anomaly-only branch is optimized to model abnormal regions. During inference, it is periodically activated within a designated timestep interval and integrated with the main anomaly-aware branch, thereby preserving fine-grained anomalies throughout the entire denoising trajectory.
\item To guide SIAS at mask-designated positions and handle anomaly–background transitions, we propose the Explicit Mask Alignment (EMA) strategy. By progressively adjusting the anomaly mask over time, EMA enhances spatial coherence and edge alignment, thereby improving downstream segmentation performance.

\end{itemize}

The remainder of this paper is structured as follows. Section~\ref{sec:related} reviews prior work on industrial anomaly synthesis and segmentation. Section~\ref{sec:method} presents the proposed STAGE framework in detail, highlighting its key components. Section~\ref{sec:experiments} outlines the experimental setup and reports extensive comparisons, ablation studies on multiple datasets, and corresponding analyses. Finally, Section~\ref{sec:conclusion} concludes this study.

\section{Related Work}\label{sec:related}

\subsection{Anomaly Synthesis}
Anomaly synthesis is widely used to compensate for the scarcity of abnormal samples. Training-free methods \cite{lyu2024reb, zhang2023destseg} rely on hand-crafted approaches (e.g., cropping and pasting anomalies), but often lack realism. GAN-based models such as SDGAN \cite{niu2020defect}, Defect-GAN \cite{zhang2021defect}, and ConGAN \cite{du2022new} improve realism yet require large abnormal datasets and offer limited controllability. Nevertheless mask-guided variants such as DFMGAN \cite{dfmgan} still suffer from poor alignment.
Recently, diffusion-based methods \cite{dai2024seas, sun2024cut, eccv2025} have shown improved realism via explicit sampling and cross-modal guidance. However, most of them are computationally expensive and designed for detection rather than segmentation. In contrast, STAGE integrates anomaly inference, graded diffusion, and EMA, enabling efficient, segmentation-oriented anomaly synthesis.

\begin{table}[ht]
  \centering
  \scriptsize
   \caption{Key notations and descriptions in STAGE.}
  \resizebox{0.48\textwidth}{!}{
  \begin{tabular}{ll}
    \toprule
    \textbf{Symbol} & \textbf{Description} \\
    \midrule
    $T$                  & Total number of diffusion steps \\
    $\mathbf{x}^{back}$  & Latent representing the background \\
    $\mathbf{x}_0$       & Latent representation encoded from raw images \\
    $\mathbf{x}^p$       & Latent representing anomaly content \\
    $\mathbf{\hat{x}}_0$ & Synthesized latent approximating raw-image representation \\
    $\mathbf{\hat{x}}^p$ & Synthesized latent encoding anomaly content \\
    $M_0$                & Ground-truth masks for abnormal raw images \\
    $E$                  & Textual embedding latent \\
    $\mathbf{t}_s$       & EMA threshold step \\
    $\mathbf{x}_t$       & Noisy latent at time step $t$ \\
    $\mathbf{x}^p_t$     & Noisy latent of anomaly content at time step $t$ \\
    $\mathbf{x}^{back}_t$& Noisy latent of normal background at time step $t$ \\
    $M_{t}$              & Noisy anomaly mask at time step $t$ \\
    $M_t^p$              & Anomaly mask at step $t$ under the EMA strategy \\
    $\beta_t$            & Variance of added noise at time step $t$ \\
    $\boldsymbol{\epsilon}$ & Standard Gaussian noise vector \\
    $\alpha_t,\bar{\alpha}_t$ 
                         & Where $\alpha_t{=}1{-}\beta_t,\;\bar{\alpha}_t{=}\prod_{s=1}^t\alpha_s$ \\
    $q(\boldsymbol{x}_t|\boldsymbol{x}_{t-1})$
                         & Forward diffusion process for raw-image latents \\
    $p_\theta(\boldsymbol{x}_{t-1}|\boldsymbol{x}_t)$
                         & Reverse denoising process for raw-image latents \\
    $q(\mathbf{x}^{back}_{t-1}|\mathbf{x}^{back})$
                         & Forward diffusion process for background latents \\
    $q(M_{t-1}|M_0)$     & Forward diffusion process for ground-truth masks \\
    $p_\theta(\mathbf{x}^p_{t}|\mathbf{x}^p_{t+1})$
                         & Reverse denoising process for anomaly latents \\
    \bottomrule
  \end{tabular}
  }
  \label{symbol}
\end{table}
\vspace{-6pt}

\begin{figure*}[t]
\centering
\includegraphics[width=1\textwidth]{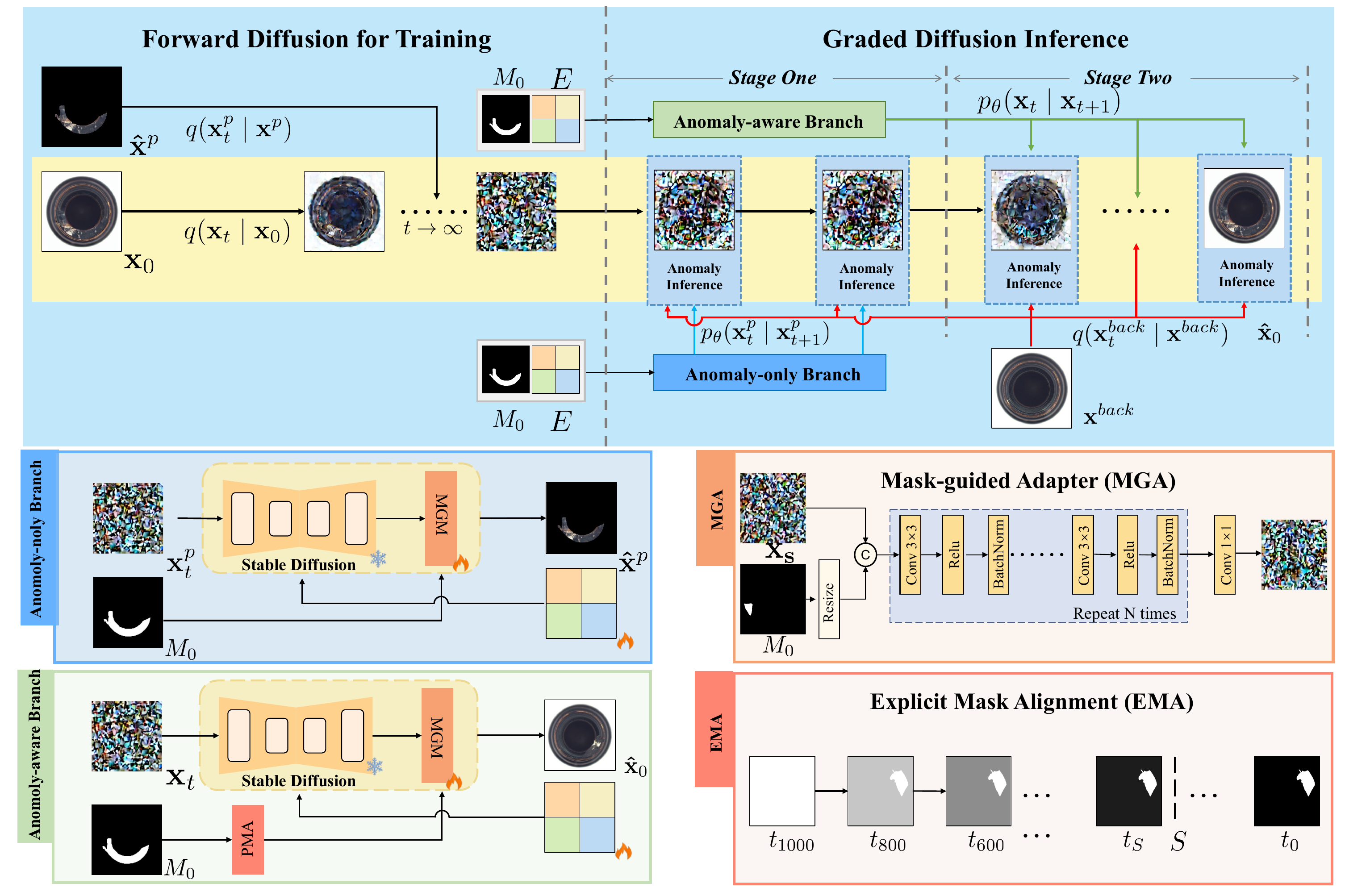}
\caption{
The architecture of the proposed STAGE framework. 
1) \textbf{Anomaly Inference}: Clean background information is incorporated as a condition in the denoising process, enabling STAGE to focus more precisely on foreground anomalies.
2) \textbf{Graded Diffusion}: The anomaly-aware and anomaly-only branches alternately model abnormal regions during denoising, ensuring that tiny anomalies are not overlooked.
3) \textbf{Explicit Mask Alignment}: A time-dependent soft mask progressively converges from a uniform matrix to the ground-truth annotation, allowing the model to leverage contextual information for realistic anomaly synthesis.
4) \textbf{Mask Guidance Adapter}: The anomaly mask serves as one of the inputs to the model, guiding it to refine the anomaly synthesis process.}
    \label{net}
\end{figure*}

\subsection{Anomaly Segmentation}
Anomaly segmentation aims to localize abnormal regions at the pixel level, whereas anomaly detection focuses on image- or region-level anomaly presence, and the two are conceptually related~\cite{qu2025dictas}. AD methods are commonly grouped into reconstruction- and embedding-based approaches~\cite{xu2023fascinating}. Reconstruction-based methods~\cite{recon1,recon2} rely on higher reconstruction errors for anomalies but often overfit, inadvertently reconstructing abnormal regions. Embedding-based methods~\cite{liu2024dual,yang2024slsg,chen2024progressive} detect anomalies as deviations from normal feature distributions, e.g., GLASS~\cite{glass} perturbs embeddings via gradient ascent to highlight weak anomalies. However, both approaches struggle with boundary precision. Recent synthesis-driven work, such as Synth4Seg~\cite{mou2024synth4seg}, employs bi-level optimization but is limited by simplistic CutPaste strategies. In contrast, STAGE targets segmentation-oriented synthesis by coupling realistic texture generation with enhanced pixel-level delineation.

\section{Proposed Method}\label{sec:method}
The proposed STAGE is built upon Latent Diffusion Models (LDM)~\cite{rombach2022high} and incorporates a dual-branch architecture to address anomaly synthesis through three core strategies, as illustrated in Fig.~\ref{net}. (\romannumeral1) \textbf{Anomaly Inference}: By conditioning the reverse diffusion process on clean background information, STAGE suppresses redundant background reconstruction and enhances attention within masked regions, enabling more focused and spatially accurate synthesis. (\romannumeral2) \textbf{Graded Diffusion}: Given that anomalies often occupy only a tiny fraction of the image, STAGE employs a dual-branch denoising strategy, where an anomaly-only branch is optimized to model the abnormal regions during training. During inference, the two branches alternate in a scheduled manner to ensure tiny anomalies are not suppressed by dominant background signals. (\romannumeral3) \textbf{Explicit Mask Alignment (EMA)}: To enable smooth foreground–background interaction, STAGE introduces a time-dependent mask that progressively sharpens toward the ground-truth anomaly mask. This mechanism improves boundary alignment and spatial coherence, thereby enhancing segmentation performance. For textual embedding, we follow the method and parameter settings from AnomalyDiffusion~\cite{hu2024anomalydiffusion}. Finally, all notations used throughout this paper are summarized in Table~\ref{symbol}.

\subsection{Anomaly Inference}
\label{sec:anomaly_inference}

In standard DDPM, the reverse distribution $p_\theta(\mathbf{x}_{t-1} \mid \mathbf{x}_t)$ is approximated by the posterior $q(\mathbf{x}_{t-1} \mid \mathbf{x}_t, \mathbf{x}_0)$, where the unknown $\mathbf{x}_0$ is replaced by the network prediction $\hat{\mathbf{x}}_0$:
\begin{equation}
\label{eq:add}
p_\theta(\mathbf{x}_{t-1} \mid \mathbf{x}_t) := q(\mathbf{x}_{t-1} \mid \mathbf{x}_t, \hat{\mathbf{x}}_0).
\end{equation}
The posterior $q(\mathbf{x}_{t-1} \mid \mathbf{x}_t, \mathbf{x}_0)$ remains Gaussian, with mean and variance:
\begin{equation}
\label{eq:sigma_t}
\begin{split}
\mu_{t-1}^\theta 
&= \frac{\sqrt{\bar{\alpha}_{t-1}} \beta_t}{1 - \bar{\alpha}_t} \mathbf{x}_0  
 + \frac{\sqrt{\alpha_t}(1 - \bar{\alpha}_{t-1})}{1 - \bar{\alpha}_t} \mathbf{x}_t, \\
\sigma_{t-1}^2 
&= \frac{1 - \bar{\alpha}_{t-1}}{1 - \bar{\alpha}_t} \beta_t ,
\end{split}
\end{equation}
where $\alpha_t = 1 - \beta_t$ and $\bar{\alpha}_t = \prod_{s=1}^{t} \alpha_s$.

Different from commonly used methods such as Repaint, we incorporate known background information by explicitly decomposing the image $\mathbf{x}_0$ into anomaly and background components, and adjust the mean of the denoising distribution:
\begin{equation}
\label{eq:decomp}
\mathbf{x}_0 = M_0 \odot \mathbf{x}^p + (1 - M_0) \odot \mathbf{x}^{\text{back}},
\end{equation}
where $M_0$ is a binary mask localizing anomalies, $\mathbf{x}^{p}$ denotes latent anomaly content, and $\mathbf{x}^{\text{back}}$ is the background from a normal image. During inference, we replace $\mathbf{x}^p$ with the model prediction $\hat{\mathbf{x}}^p$ while keeping $\mathbf{x}^{\text{back}}$ fixed.

Following Eq.~\eqref{eq:sigma_t}, the conditional mean $\mu_{t-1}^\theta$ is re-written as follows:
\begin{equation}
\label{eq:mu_mix}
\begin{split}
\mu_{t-1}^\theta 
&= \frac{\sqrt{\bar{\alpha}_{t-1}} \beta_t}{1 - \bar{\alpha}_t} 
   (\hat{\mathbf{x}}^p + \mathbf{x}^{\text{back}}) \\
&\quad + \frac{\sqrt{\alpha_t} (1 - \bar{\alpha}_{t-1})}{1 - \bar{\alpha}_t} 
   (\mathbf{x}_t^p + \mathbf{x}_t^{\text{back}}) \\
&= M_0 \odot \mu_{t-1}^{\hat{\mathbf{x}}^p} 
   + (1 - M_0) \odot \mu_{t-1}^{\mathbf{x}^{\text{back}}},
\end{split}
\end{equation}
where the sub-region means are:
\begin{equation}
\begin{split}
\mu_{t-1}^{\hat{\mathbf{x}}^p} 
&= \frac{\sqrt{\bar{\alpha}_{t-1}} \beta_t}{1 - \bar{\alpha}_t}\, \hat{\mathbf{x}}^p 
 + \frac{\sqrt{\alpha_t}(1 - \bar{\alpha}_{t-1})}{1 - \bar{\alpha}_t}\, \mathbf{x}_t^p, \\
\mu_{t-1}^{\mathbf{x}^{\text{back}}} 
&= \frac{\sqrt{\bar{\alpha}_{t-1}} \beta_t}{1 - \bar{\alpha}_t}\, \mathbf{x}^{\text{back}} 
 + \frac{\sqrt{\alpha_t}(1 - \bar{\alpha}_{t-1})}{1 - \bar{\alpha}_t}\, \mathbf{x}_t^{\text{back}}.
\end{split}
\end{equation}
Here, $\mathbf{x}_t^p = M_0 \odot \mathbf{x}_t$ and $\mathbf{x}_t^{\text{back}} = (1 - M_0) \odot \mathbf{x}_t$ denote anomaly and background components at step $t$. Assuming statistical independence of noise across disjoint regions, we treat them as conditionally independent in reverse inference. The full reverse denoising distribution is then a spatial mixture of two Gaussians:
\begin{equation}
\label{eq:masked_mixture}
\begin{split}
p_\theta(\mathbf{x}_{t-1} \mid \mathbf{x}_t) 
&= M_0 \odot \mathcal{N}\!\left(\mu_{t-1}^{\hat{\mathbf{x}}^p}, \, \sigma_{t-1}^2 \mathbf{I}\right) \\
&\quad + (1 - M_0) \odot \mathcal{N}\!\left(\mu_{t-1}^{\mathbf{x}^{\text{back}}}, \, \sigma_{t-1}^2 \mathbf{I}\right).
\end{split}
\end{equation}
It encodes spatial independence via the binary mask $M_0$, ensuring a clean separation between anomaly synthesis and background recovery. By decoupling anomaly modeling from background reconstruction, STAGE concentrates generative capacity on anomalous regions while preserving high-fidelity normal structure. It guides the sampling trajectory away from regenerating the background and toward generating plausible anomalies for SIAS.

\subsection{Graded diffusion}\label{sec:phased_diffusion}
\begin{algorithm}[ht]
  \caption{Graded Diffusion}
  \label{PDiffusion}
  \centering
  \scriptsize
  \begin{algorithmic}
    \STATE \textbf{Input:} $M_0,\;x^{back}$
    \STATE \textbf{Output:} $\hat{\mathbf{x}}_0$
    \STATE Initialize $\mathbf{x}_T\sim\mathcal{N}(0,\mathbf{I}),\;M^p_T=\mathbf{I},\;T=1000$
    \FOR{$t=T$ \textbf{down to} $0$}
      \STATE $M^p_{t-1}\gets \mathrm{EMA}(M^p_t,t)$; \quad Sample $\mathbf{x}^{back}_{t-1}\sim q(\mathbf{x}^{back}_{t-1}\mid \mathbf{x}^{back})$
      \STATE Sample $M_{t-1}\sim q(M_{t-1}\mid M_0)$
      \IF{$t\in[d_1,d_2]$}
        \STATE Sample $\hat{\mathbf{x}}^p_{t-1}$ from anomaly-only branch: \quad $\hat{\mathbf{x}}^p_{t-1}=M_0\odot\hat{\mathbf{x}}^p_{t-1}+(1-M_0)\odot M_{t-1}$
        \STATE $\hat{\mathbf{x}}_{t-1}=M_0\odot\hat{\mathbf{x}}^p_{t-1}+(1-M_0)\odot\mathbf{x}^{back}_{t-1}$ 
      \ELSE
        \STATE Sample $\hat{\mathbf{x}}_{t-1}$ from anomaly-aware branch: \quad $\hat{\mathbf{x}}_{t-1}=M^p_{t-1}\odot\hat{\mathbf{x}}_{t-1}+(1-M^p_{t-1})\odot\mathbf{x}^{back}_{t-1}$
        \STATE $\hat{\mathbf{x}}^p_{t-1}=M_0\odot\hat{\mathbf{x}}_{t-1}+(1-M_0)\odot M_{t-1}$
      \ENDIF
    \ENDFOR
  \end{algorithmic}
\end{algorithm}
In practical industrial scenarios, background regions often dominate the image, while anomalies occupy only a small fraction. This severe imbalance causes anomaly features to be easily overlooked during inference, leading to misalignment between synthesized content and anomaly masks. To mitigate this, we propose \textit{Graded Diffusion}, a dual-branch denoising strategy in which global context and localized anomaly features are learned separately during training and dynamically fused during inference through graded sampling. As illustrated in Algorithm~\ref{PDiffusion}, the two branches are trained and sampled in a coordinated manner to ensure that anomaly regions receive sufficient generative focus.\\
During training, the anomaly latent $\mathbf{x}_0$ is fed into the \textit{anomaly-aware branch}, which learns to reconstruct the entire image with contextual consistency. Simultaneously, an anomaly-only input $\mathbf{x}^p$ is constructed by masking $\mathbf{x}_0$ with its annotation $M_0$, thereby isolating the abnormal region and suppressing the background. This masked latent is then used to train the \textit{anomaly-only branch}, which focuses exclusively on learning fine-grained abnormal features. During inference, to prevent the loss of anomaly information, the anomaly-only branch is activated within a designated interval $\left[d_1, d_2\right]$ to generate localized anomaly content, particularly when it is most likely to be suppressed by dominant background features. The synthesized anomaly content is merged with the known background to ensure spatial alignment with the target anomaly regions. When outside this interval, the anomaly-aware branch governs the sampling process, and soft fusion is applied via the progressive mask $M^p_t$ (introduced in Sec.~\ref{sec:EMA}). The overall update rule for the synthesized sample at each timestep is given by:
\begin{align}
    \hat{\mathbf{x}}_{t-1} &=
    \begin{cases}
    M_0 \odot \hat{\mathbf{x}}^p_{t-1} + (1 - M_0) \odot \mathbf{x}^{\mathrm{back}}_{t-1}, \\
     \text{if } t \in [d_1, d_2], \\[6pt]
    M^p_t \odot \hat{\mathbf{x}}_{t-1} + (1 - M^p_{t-1}) \odot \mathbf{x}^{\mathrm{back}}_{t-1}, 
    \\ \text{otherwise}.
    \end{cases} \notag \\
    \hat{\mathbf{x}}^p_{t-1} &=
    \begin{cases}
    \text{output from anomaly-only branch}, \\
    \text{if } t \in [d_1, d_2], \\[6pt]
    M_0 \odot \hat{\mathbf{x}}_{t-1} + (1 - M_0) \odot M_{t-1}, \\
    \text{otherwise}.
    \end{cases}
    \label{xp}
\end{align}


Meanwhile, to maintain the continuity of anomaly representation across all timesteps, we also update $\hat{\mathbf{x}}^p_{t-1}$ at each step in the following Eq.~\eqref{xp}. It ensures that even when the anomaly-only branch is inactive, the corresponding latent anomaly region continues to update, enabling temporally consistent anomaly synthesis. By alternating between anomaly-aware and anomaly-only branches during inference, this graded sampling strategy effectively prevents the model from neglecting small or low-contrast anomalies, which would otherwise be unlikely to persist or accumulate across denoising steps, thereby ensuring that subtle anomaly signals are continuously preserved and gradually reinforced throughout the generation process.

\subsection{Explicit Mask Alignment} \label{sec:EMA}
Neglecting the surrounding context often leads to semantic mismatches and implausible structures in reconstructed abnormal regions~\cite{lugmayr2022repaint}. To address this issue, we propose the \textbf{Explicit Mask Alignment (EMA)} strategy, which modulates the fusion of contextual and anomalous information throughout the denoising process, as illustrated in Fig.~\ref{net}. Specifically, since low-frequency components such as global structure are typically restored earlier in diffusion-based denoising, while high-frequency details (e.g., boundaries and textures) emerge in later steps~\cite{song2021scorebased}, EMA employs an adaptive mask that gradually transitions from an all-one matrix to the real anomaly mask as the time step $t$ decreases. This design ensures that early steps focus more on preserving the global context, while later steps progressively emphasize reconstructing the anomalous regions with precise details. Formally, we initialize the mask as an all-one matrix $M^f$, and define its linear transition toward the mask annotation $M_0$ via:

\begin{equation}
\label{eq:mask_schedule}
\begin{aligned}
M^p_t &= \zeta(t)\, M^c + M_0, \\
M^c &= \mathbf{1} - M_0,\\
\zeta(t) &=
\begin{cases}
\dfrac{t - t_s}{T - t_s}, & t > t_s,\\[3pt]
0, & t \le t_s.
\end{cases}
\end{aligned}
\end{equation}

And it allows contextual information to dominate in early steps, while gradually converging to $M_0$ as $t \leq t_s$, thereby enabling more precise and localized SIAS. To analyze EMA’s effectiveness, we construct a pixel-wise optimization model, where $\hat{x}_{0,i}^p(t)$ and $\hat{x}_{0,i}^{\mathrm{back}}(t)$ denote the outputs of the network at pixel $i$ and timestep $t$, respectively. Then $\hat{x}_{0,i}(w)$ is fused using a weight $w \in [0,1]$:

\begin{equation}
\hat{x}_{0,i}(w) = w\,\hat{x}^p_{0,i}(t) + (1 - w)\,\hat{x}^{\mathrm{back}}_{0,i}(t),
\end{equation}
where $w = M_{0,i}$ corresponds to the hard binary mask, and $w = M^p_{t,i}$ corresponds to our EMA scheme. Since the error is defined as:
\begin{equation}
E_i(w, t) = \left[ w\,\delta_{p,i}(t) + (1 - w)\,\delta_{b,i}(t) \right]^2,
\label{eq:pixel_error}
\end{equation}
where $\delta_{p,i}(t) = \hat{x}^p_{0,i}(t) - x^p_{0,i}$ and $\delta_{b,i}(t) = \hat{x}^{\mathrm{back}}_{0,i}(t) - x^{\mathrm{back}}_{0,i}$. In our EMA scheme, the fusion weight for pixel \(i\) at time \(t\) is given by
\[
w_i^{\mathrm{EMA}} := M^p_{t,i},
\]
where \(M^p_{t,i}\) is defined in Eq.~\eqref{eq:mask_schedule}. This corresponds to the soft mask coefficient used for blending predictions from the anomaly and background branches at pixel \(i\). Then the superiority of EMA over binary masks is demonstrated through Theorem~\ref{lemma:ema_optimality}.

\begin{theorem}[Near-optimality of EMA]\label{lemma:ema_optimality}
Under the assumptions that the per-pixel prediction errors
\(\delta_{p,i}(t)\) and \(\delta_{b,i}(t)\) are \(L\)-Lipschitz‐continuous in \(t\)
and the total number of diffusion steps \(T\) is sufficiently large, the Explicit Mask Alignment (EMA)
strategy incurs only a vanishing excess reconstruction error compared to the per-step optimal weights \(w_i^*(t)\).

Concretely, if the deviation of the linear EMA schedule from the optimal weight satisfies
\[
\epsilon_i(t) \;=\;\bigl|w_i^*(t) - M^p_{t,i}\bigr|\;\le\;\frac{C}{T}
\quad\text{for some constant }C,
\]
then the total error increase over all pixels admits the bound
\begin{equation}
\begin{aligned}
&\sum_i \big(E_i(w_i^{\mathrm{EMA}},t)-E_i(w_i^*(t),t)\big)
\le \\
&\sum_i (\delta_{p,i}(t)-\delta_{b,i}(t))^2\,\epsilon_i(t)^2 := \mathcal{O}(T^{-2}).
\end{aligned}
\end{equation}

Hence, EMA achieves a near-optimal reconstruction error with the excess bounded by \(\mathcal{O}(1/T^2)\).
\end{theorem}

\begin{proof}
We proceed in two steps: (1) derive the per‐pixel optimal weight and show EMA strictly improves over a binary mask; (2) bound the deviation of the EMA schedule from this optimum and its effect on total error.

\textbf{Step 1.}
For fixed pixel \(i\) and timestep \(t\), define
\[
E_i(w,t)
=\bigl[w\,\delta_{p,i}(t)+(1-w)\,\delta_{b,i}(t)\bigr]^2,
\]
where 
\(\delta_{p,i}(t)=\hat x^p_{0,i}(t)-x^p_{0,i}\) and
\(\delta_{b,i}(t)=\hat x^{\mathrm{back}}_{0,i}(t)-x^{\mathrm{back}}_{0,i}\).
Since \(E_i\) is a convex quadratic in \(w\), its unique minimizer is obtained by
\begin{equation*}
\begin{split}
\frac{\partial E_i}{\partial w}
&=2\bigl[w\,\delta_{p,i}+(1-w)\,\delta_{b,i}\bigr]\cdot(\delta_{p,i}-\delta_{b,i}) \\
&=0 \;\Longrightarrow\;
w_i^*(t)=\frac{-\delta_{b,i}(t)}{\delta_{p,i}(t)-\delta_{b,i}(t)}.
\end{split}
\end{equation*}

If \(\delta_{p,i}(t)\) and \(\delta_{b,i}(t)\) have opposite signs, then \(w_i^*(t)\in(0,1)\)
and by convexity,
\[
E_i\bigl(w_i^*(t),t\bigr)
<\min\{E_i(0,t),\,E_i(1,t)\}.
\]
Otherwise the minimizer lies at \(w=0\) or \(w=1\), to which the EMA schedule
\(M^p_{t,i}\) also converges as \(t\) decreases. Thus EMA never exceeds the error of
a binary mask, and strictly improves on boundary pixels where \(w_i^*(t)\in(0,1)\).

\textbf{Step 2.}
By assumpting \(\delta_{p,i}(t)\) and \(\delta_{b,i}(t)\) are \(L\)-Lipschitz‐continuous in \(t\).  Define the mapping
\[
g(x,y) \;=\; -\frac{y}{x-y}.
\]
Since \(g\) is smooth whenever \(x\!-\!y\) is bounded away from zero, the composition
\[
w_i^*(t) = g\bigl(\delta_{p,i}(t),\,\delta_{b,i}(t)\bigr)
\]
is also Lipschitz‐continuous.  Hence there exists a constant \(L'\) such that
\[
\bigl|w_i^*(t_1)-w_i^*(t_2)\bigr|
\;\le\;
L'\,\bigl|t_1-t_2\bigr|
\quad\forall\,t_1,t_2.
\]

Meanwhile, the linear EMA schedule \(\zeta(t)\) decreases by at most \(1/(T-t_s)\le1/T\)
per timestep, so the deviation
\[
\epsilon_i(t)
=\bigl|w_i^*(t)-M^p_{t,i}\bigr|
\;\le\;\frac{L'}{T}
\;\equiv\;\frac{C}{T}.
\]
Using the convexity bound
\[
E_i(M^p_{t,i},t)-E_i(w_i^*(t),t)
\;\le\;(\delta_{p,i}(t)-\delta_{b,i}(t))^2\,\epsilon_i(t)^2,
\]
and summing over all pixels \(i\), we obtain
\begin{multline*}
\sum_i \bigl[E_i(w_i^{\mathrm{EMA}},t)-E_i(w_i^*(t),t)\bigr] \;\le\; \\[-0.3ex]
\sum_i (\delta_{p,i}(t)-\delta_{b,i}(t))^2\,\frac{C^2}{T^2}
=\;\mathcal{O}\!\Bigl(\tfrac{1}{T^2}\Bigr).
\end{multline*}

This completes the proof that EMA’s excess reconstruction error per step is
\(\mathcal{O}(1/T^2)\), which vanishes as \(T\to\infty\).
\end{proof}


For DDPMs with a sufficiently large number of timesteps ($T \gg 1$), the change between consecutive steps becomes infinitesimally small. As a result, the latent states $x_t$ and $x_{t-1}$ are highly correlated, and the neural network receives only small perturbations in its input across timesteps in practice. \textbf{This smoothness has been empirically observed and theoretically assumed in numerous diffusion-based models~\cite{song2021scorebased,lugmayr2022repaint}. Therefore, assuming that the prediction error terms $\delta_{p,i}(t)$ and $\delta_{b,i}(t)$ are Lipschitz-continuous is a realistic approximation for large-$T$ DDPM models and has been widely recognized} \cite{lee2023convergence,dou2024theory,lyu2024sampling}.  In addition, EMA allows intermediate fusion weights $w \in (0,1)$ where optimal, improvises SIAS over traditional binary masks, and achieves a near-optimal error in an average sense, with provable $\mathcal{O}(1/T^2)$ convergence to the optimal trajectory under standard smoothness assumptions.

\subsection{Mask Guidance Adapter}
To improve spatial alignment for downstream tasks via SIAS, STAGE introduces \emph{Mask Guidance Adapter (MGA)}, as illustrated in Fig.~\ref{net}. MGA is specifically designed to fine-tune the denoising network, without altering the original architecture. It leverages the anomaly mask $M_0$ as prior information to modulate features, encouraging synthesized anomalies to appear only within designated regions. Specifically, after obtaining the output from the frozen denoising network, we concatenate it with the corresponding mask ($M_0$ for the anomaly-only branch, and $M_t^p$ otherwise). The concatenated features are then processed by $N$ repeated blocks consisting of convolution, ReLU, and batch normalization, enabling precise anomaly synthesis aligned with the mask.

\subsection{Loss function}
\label{loss}

Since the anomaly-only branch in STAGE is designed to retain and highlight anomaly information within the masked regions, it is essential to explicitly enforce anomaly generation within these regions. Meanwhile, the anomaly-aware branch integrates contextual information via the EMA strategy, thereby enhancing spatial coherence and anomaly representation. To accommodate the dual-branch architecture, STAGE modifies the standard diffusion loss as follows:
\begin{align}
\label{loss}
\mathcal{L}_{\text{p}} &= \mathbb{E}_{\mathbf{x}_0, \boldsymbol{\epsilon}_1, t} \left[ \left\| \boldsymbol{\epsilon}_1 - \boldsymbol{\epsilon}_{ao}\left(\mathbf{x}_t, t\right) \right\|_{L1} \right] \\
&\quad + \mathbb{E}_{\mathbf{x}_0^p, \boldsymbol{\epsilon}_2, t} \left[ \left\| M^p_t \odot \left( \boldsymbol{\epsilon}_2 - \boldsymbol{\epsilon}_{aa}\left(\mathbf{x}_t^p, t\right) \right) \right\|_{L1} \right], \notag
\end{align}
where $\boldsymbol{\epsilon}_1, \boldsymbol{\epsilon}_2 \sim \mathcal{N}(0,\mathbf{I})$, $\boldsymbol{\epsilon}_{ao}$ and $\boldsymbol{\epsilon}_{aa}$ are the predicted noises from the anomaly-only and anomaly-aware branches, respectively, and $M^p_t$ denotes the EMA-modulated spatial mask at time step $t$.

\section{Experiments}
\label{sec:experiments}

\subsection{Implementations} 

We conduct experiments on two public anomaly datasets, \textbf{MVTec AD}~\cite{bergmann2019mvtec} and \textbf{BTAD}~\cite{btad}. For each anomaly type, we generate synthetic anomaly samples using normal images, masks, and text annotations. Specifically, for \textbf{MVTec AD} we set the prompt to \emph{``class name + anomaly type''}, while for \textbf{BTAD}, where class names are numerical IDs, we uniformly set the prompt to \emph{``damaged''}. In total, 500 anomaly image–mask pairs are generated per type, with about one-third of the real images used for training the segmentation model and the remaining two-thirds reserved for testing.  

We compare our proposed STAGE against six representative SIAS baselines: CutPaste~\cite{li2021cutpaste}, DRAEM~\cite{zavrtanik2021draem}, GLASS~\cite{glass}, DFMGAN~\cite{dfmgan}, AnomalyDiffusion~\cite{hu2024anomalydiffusion}, and RealNet~\cite{zhang2024realnet}. For downstream segmentation, we further evaluate three lightweight networks suitable for real-time industrial applications: SegFormer~\cite{xie2021segformer}, BiSeNet V2~\cite{yu2021bisenet}, and STDC~\cite{stdc}.
For evaluation, we adopt mean Intersection over Union (\textbf{mIoU}) to measure the overlap between predicted segmentation and ground truth, and pixel accuracy (\textbf{Acc}) to measure the proportion of correctly classified pixels.
 In addition to these segmentation-oriented metrics, we also report detection-oriented metrics including AUROC, PRO, F1-score, and Average Precision (AP) as auxiliary references. Higher values of these indicators indicate better model performance.

In terms of hyperparameters, we set $T=1000$ and $t_s=200$ for EMA. In graded diffusion, the anomaly-only branch is activated when $t$ falls within $[1000,800]$ and $[400,300]$. The textual embedding $E$ consists of 8 tokens. We use a batch size of 4 and a learning rate of 0.001. All experiments are conducted on 4 NVIDIA A100 80GB GPUs and trained for approximately 80,000 iterations.

\begin{table*}[htbp]
\centering
\caption{Comparison of pixel-level anomaly segmentation (mIoU/Acc) using the real-time SegFormer trained on synthetic MVTec data produced from the proposed STAGE and other existing SIAS methods.}
\resizebox{1\textwidth}{!}{
\begin{tabular}{l|cc|cc|cc|cc|cc|cc|cc}
\hline
\rowcolor{blue!10}
\textbf{Category} 
& \multicolumn{2}{c|}{\textbf{CutPaste}} 
& \multicolumn{2}{c|}{\textbf{DRAEM}} 
& \multicolumn{2}{c|}{\textbf{GLASS}} 
& \multicolumn{2}{c|}{\textbf{RealNet}} 
& \multicolumn{2}{c|}{\textbf{DFMGAN}} 
& \multicolumn{2}{c|}{\textbf{AnomalyDiffusion}} 
& \multicolumn{2}{c}{\textbf{STAGE}} \\
\hline
\rowcolor{blue!10}
& \textbf{mIoU} $\uparrow$ & \textbf{Acc} $\uparrow$ 
& \textbf{mIoU} $\uparrow$ & \textbf{Acc} $\uparrow$ 
& \textbf{mIoU} $\uparrow$ & \textbf{Acc} $\uparrow$ 
& \textbf{mIoU} $\uparrow$ & \textbf{Acc} $\uparrow$ 
& \textbf{mIoU} $\uparrow$ & \textbf{Acc} $\uparrow$ 
& \textbf{mIoU} $\uparrow$ & \textbf{Acc} $\uparrow$ 
& \textbf{mIoU} $\uparrow$ & \textbf{Acc} $\uparrow$ \\
\hline

bottle     & 75.11 & 79.49 & 79.51 & 84.99 & 70.26 & 76.30 & 77.96 & 83.90 & 75.45 & 80.39 & 76.39 & 83.54 & \textbf{90.07} & \textbf{93.81} \\
cable      & 55.40 & 60.49 & 64.52 & 70.77 & 58.81 & 62.32 & 62.51 & 69.27 & 62.10 & 64.87 & 62.49 & 74.48 & \textbf{79.04} & \textbf{87.49} \\
capsule    & 35.15 & 40.29 & 51.39 & 62.32 & 34.12 & 38.04 & 46.76 & 51.91 & 41.29 & 15.83 & 37.73 & 44.72 & \textbf{59.96} & \textbf{70.05} \\
carpet     & 66.34 & 77.59 & 72.57 & 81.28 & 70.11 & 77.56 & 68.84 & 79.15 & 71.33 &  \textbf{83.69} & 64.67 & 73.59 & \textbf{73.23} & {83.24} \\
grid       & 29.90 & 46.72 & \textbf{47.75} &  \textbf{67.85} & 37.43 & 46.30 & 37.55 & 48.86 & 37.73 & 54.13 & 38.70 & 51.82 & {46.32} & {62.37} \\
hazel\_nut & 56.95 & 60.72 & 84.22 & 89.74 & 55.51 & 57.43 & 60.18 & 63.49 & 83.43 & 86.03 & 59.33 & 67.48 & \textbf{88.10} & \textbf{92.84} \\
leather    & 57.23 & 63.49 & 64.12 & 71.49 & 62.05 & 73.38 & 68.29 & 77.16 & 60.96 & 68.02 & 56.45 & 62.51 & \textbf{72.42} & \textbf{81.68} \\
metal\_nut & 88.78 & 90.94 & 93.51 & 96.10 & 88.15 & 90.52 & 91.28 & 94.09 & 92.77 & 94.93 & 88.00 & 91.10 & \textbf{95.00} & \textbf{97.87} \\
pill       & 43.28 & 47.11 & 46.99 & 49.76 & 41.52 & 43.54 & 47.32 & 58.31 & 87.19 & 90.05 & 83.21 & 89.00 & \textbf{90.81} & \textbf{97.06} \\
screw      & 25.10 & 31.35 & 46.96 & 59.03 & 35.94 & 42.37 & 47.12 & 55.17 & \textbf{46.65} & 50.79 & 38.47 & 49.49 & {42.42} & \textbf{57.36} \\
tile       & 85.33 & 91.60 & 89.21 & 93.74 & 85.67 & 90.28 & 83.53 & 87.30 & 88.87 & 91.96 & 84.29 & 89.72 & \textbf{90.76} & \textbf{93.56} \\
toothbrush & 39.40 & 63.93 & 65.35 & 79.43 & 53.75 & 60.46 & 57.68 & 72.03 & 61.00 & 70.50 & 48.68 & 64.41 & \textbf{70.77} & \textbf{88.49} \\
transistor & 65.03 & 71.05 & 59.96 & 62.18 & 29.28 & 30.67 & 63.71 & 66.79 & 73.56 & 78.48 & 79.27 &  \textbf{91.74} & \textbf{79.68} & {85.06} \\
wood       & 49.64 & 60.47 & 67.52 & 73.28 & 50.91 & 53.16 & 61.84 & 89.54 & 67.00 & 80.84 & 60.16 & 74.62 & \textbf{75.94} & \textbf{87.14} \\
zipper     & 65.39 & 71.89 & 69.29 & 79.36 & 69.98 & 79.31 & 68.78 & 78.50 & 66.34 & 70.50 & 65.36 & 72.66 & \textbf{77.22} & \textbf{84.12} \\
\hline

\rowcolor{blue!10}
Average    & 55.87 & 63.81 & 66.86 & 74.75 & 56.23 & 61.44 & 62.89 & 71.70 & 67.71 & 72.07 & 62.88 & 72.06 & \textbf{75.45} & \textbf{84.14} \\
\hline
\end{tabular}
}
\label{Segformer_mvtec}
\end{table*}

\begin{table*}[htbp]
\centering
\caption{Comparison of pixel-level anomaly segmentation (mIoU/Acc) using the real-time BiseNet V2 trained on synthetic MVTec data produced from the proposed STAGE and other existing AS methods.}
\resizebox{1\textwidth}{!}{
\begin{tabular}{l|cc|cc|cc|cc|cc|cc|cc}
\hline
\rowcolor{blue!10}\textbf{Category} 
  & \multicolumn{2}{c|}{\textbf{CutPaste}} 
  & \multicolumn{2}{c|}{\textbf{DRAEM}} 
  & \multicolumn{2}{c|}{\textbf{GLASS}} 
  & \multicolumn{2}{c|}{\textbf{RealNet}} 
  & \multicolumn{2}{c|}{\textbf{DFMGAN}} 
  & \multicolumn{2}{c|}{\textbf{AnomalyDiffusion}} 
  & \multicolumn{2}{c}{\textbf{STAGE}} \\ \hline
\rowcolor{blue!10}
  & \textbf{mIoU} $\uparrow$ & \textbf{Acc} $\uparrow$ 
  & \textbf{mIoU} $\uparrow$ & \textbf{Acc} $\uparrow$ 
  & \textbf{mIoU} $\uparrow$ & \textbf{Acc} $\uparrow$ 
  & \textbf{mIoU} $\uparrow$ & \textbf{Acc} $\uparrow$ 
  & \textbf{mIoU} $\uparrow$ & \textbf{Acc} $\uparrow$ 
  & \textbf{mIoU} $\uparrow$ & \textbf{Acc} $\uparrow$ 
  & \textbf{mIoU} $\uparrow$ & \textbf{Acc} $\uparrow$ \\ 
\hline
bottle      & 71.77 & 78.57 & 75.13 & 79.17 & 57.81 & 60.79 & 72.16 & 75.55 & 64.28 & 71.31 & 75.28 & 85.11 & \textbf{88.69} & \textbf{93.54} \\
cable       & 46.00 & 57.08 & 53.88 & 60.96 & 16.63 & 16.65 & 51.22 & 62.32 & 57.09 & 63.25 & 60.55 & 74.96 & \textbf{78.81} & \textbf{89.01} \\
capsule     & 25.97 & 37.04 & 36.82 & 42.19 & 19.53 & 51.89 & 35.97 & 39.39 & 28.40 & 31.18 & 26.77 & 32.87 & \textbf{49.86} & \textbf{61.25} \\
carpet      & 58.98 & 72.22 & \textbf{68.42} & \textbf{77.21} & 64.77 & 73.93 &  8.98 &  9.01 & 62.13 & 67.98 & 58.18 & 64.69 & 68.05 & 77.07 \\
grid        & 24.68 & 44.17 & \textbf{42.81} & \textbf{63.34} &  6.50 &  6.91 & 10.61 & 11.47 & 10.17 & 15.23 & 18.98 & 24.30 & 38.53 & 55.13 \\
hazel\_nut  & 47.93 & 53.57 & 74.83 & 81.35 & 71.54 & 75.62 & 60.16 & 65.93 & 79.78 & 84.37 & 57.26 & 70.41 & \textbf{84.61} & \textbf{92.20} \\
leather     & 31.11 & 58.36 & 55.07 & 61.58 & 57.98 & 71.84 & 53.77 & 63.85 & 31.77 & 34.82 & 50.02 & 61.60 & \textbf{58.65} & \textbf{76.75} \\
metal\_nut  & 82.95 & 87.73 & 91.58 & 94.73 & 83.82 & 85.42 & 88.38 & 90.73 & 91.17 & 93.57 & 85.52 & 90.20 & \textbf{92.75} & \textbf{97.02} \\
pill        & 55.62 & 67.04 & 45.23 & 48.99 & 23.88 & 24.15 & 72.59 & 86.32 & 82.40 & 84.30 & 80.87 & 87.02 & \textbf{88.15} & \textbf{95.87} \\
screw       &  4.88 &  6.63 & 25.08 & 35.77 & 12.32 & 13.11 & 22.35 & 23.78 & \textbf{38.14} & \textbf{40.36} & 23.23 & 29.91 & 28.67 & 37.27 \\
tile        & 76.25 & 85.75 & 86.17 & 90.45 & 77.32 & 80.28 & 77.16 & 84.84 & 85.69 & 90.12 & 79.32 & 85.63 & \textbf{88.04} & \textbf{91.52} \\
toothbrush  & 35.69 & 50.45 & 57.66 & 79.15 & 38.86 & 51.97 & 32.38 & 37.88 & 48.83 & 58.76 & 44.33 & 69.32 & \textbf{66.69} & \textbf{82.62} \\
transistor  & 44.48 & 51.79 & 59.88 & 65.96 & 44.93 & 53.04 & 61.68 & 68.59 & 76.52 & 82.13 & 76.34 & \textbf{89.94} & \textbf{83.03} & 89.25 \\
wood        & 35.51 & 46.00 & 49.82 & 62.09 & 36.41 & 51.10 & 47.29 & 61.35 & 51.84 & 63.70 & 52.06 & 72.75 & \textbf{68.98} & \textbf{82.98} \\
zipper      & 51.61 & 63.09 & 66.88 & 75.75 & 61.99 & 70.07 & 66.09 & 77.54 & 60.61 & 71.11 & 57.86 & 67.64 & \textbf{74.99} & \textbf{84.32} \\ 
\hline
\rowcolor{blue!10}
Average     & 46.23 & 57.30 & 59.28 & 67.91 & 44.95 & 52.45 & 50.72 & 57.24 & 57.92 & 63.48 & 56.44 & 67.09 & \textbf{70.57} & \textbf{80.39} \\ 
\hline
\end{tabular}
}
\label{mvtec_bise}
\end{table*}

\begin{table*}[htbp]
\centering
\vskip -0.1in
\caption{Comparison of pixel-level anomaly segmentation (mIoU/Acc) using the real-time STDC trained on synthetic MVTec data produced from the proposed STAGE and other existing AS methods.}
\resizebox{1\textwidth}{!}{
\begin{tabular}{l|cc|cc|cc|cc|cc|cc|cc}
\hline
\rowcolor{blue!10}\textbf{Category} 
  & \multicolumn{2}{c|}{\textbf{CutPaste}} 
  & \multicolumn{2}{c|}{\textbf{DRAEM}} 
  & \multicolumn{2}{c|}{\textbf{GLASS}} 
  & \multicolumn{2}{c|}{\textbf{RealNet}} 
  & \multicolumn{2}{c|}{\textbf{DFMGAN}} 
  & \multicolumn{2}{c|}{\textbf{AnomalyDiffusion}} 
  & \multicolumn{2}{c}{\textbf{STAGE}} \\ \hline
\rowcolor{blue!10}
  & \textbf{mIoU} $\uparrow$ & \textbf{Acc} $\uparrow$ 
  & \textbf{mIoU} $\uparrow$ & \textbf{Acc} $\uparrow$ 
  & \textbf{mIoU} $\uparrow$ & \textbf{Acc} $\uparrow$ 
  & \textbf{mIoU} $\uparrow$ & \textbf{Acc} $\uparrow$ 
  & \textbf{mIoU} $\uparrow$ & \textbf{Acc} $\uparrow$ 
  & \textbf{mIoU} $\uparrow$ & \textbf{Acc} $\uparrow$ 
  & \textbf{mIoU} $\uparrow$ & \textbf{Acc} $\uparrow$ \\ 
\hline
bottle      & 71.37 & 82.19 & 73.31 & 78.23 & 63.22 & 69.25 & 69.44 & 75.68 & 67.66 & 76.52 & 72.66 & 84.94 & \textbf{89.77} & \textbf{93.79} \\
cable       & 46.00 & 57.08 & 50.02 & 58.38 & 49.38 & 57.80 & 35.97 & 38.81 & 57.74 & 62.86 & 59.43 & 74.22 & \textbf{80.63} & \textbf{91.24} \\
capsule     & 21.73 & 30.72 & 36.31 & 41.68 & 22.91 & 27.18 & 31.08 & 34.25 & 25.60 & 27.96 & 22.90 & 26.06 & \textbf{46.65} & \textbf{55.08} \\
carpet      & 50.79 & 66.68 & \textbf{66.28} & \textbf{76.70} & 63.18 & 77.85 & 57.48 & 68.51 & 58.58 & 71.83 & 56.16 & 68.47 & \textbf{70.34} & \textbf{82.14} \\
grid        & 15.24 & 25.75 & \textbf{30.29} & \textbf{41.50} & 19.89 & 24.72 &  5.37 &  5.85 &  1.39 &  1.39 & 16.20 & 24.63 & 29.45 & 30.40 \\
hazel\_nut  & 58.48 & 65.59 & 78.75 & 83.66 & 68.57 & 85.83 & 70.16 & 82.40 & 81.77 & 84.66 & 61.83 & 92.42 & \textbf{85.23} & \textbf{93.04} \\
leather     & 38.12 & 58.63 & 44.63 & 56.84 & 57.53 & 73.90 & 36.76 & 53.88 & 21.29 & 22.28 & 46.98 & 59.89 & \textbf{61.64} & \textbf{76.44} \\
metal\_nut  & 81.13 & 86.63 & 91.12 & 94.08 & 83.97 & 89.37 & 86.85 & 91.45 & 90.68 & 92.73 & 85.81 & 90.06 & \textbf{93.11} & \textbf{97.05} \\
pill        & 50.00 & 60.28 & 55.47 & 61.05 & 44.48 & 48.11 & 63.96 & 65.96 & 80.41 & 82.55 & 78.23 & 84.35 & \textbf{85.87} & \textbf{94.17} \\
screw       &  2.80 &  4.98 & 16.16 & 23.05 & 16.81 & 19.33 & \textbf{34.93} & \textbf{38.76} & 17.93 & 18.76 &  1.27 &  2.00 & 16.60 & 19.21 \\
tile        & 69.86 & 78.18 & 84.75 & \textbf{91.31} & 79.86 & 88.65 & \textbf{85.36} & 89.72 & 70.29 & 77.70 & 76.96 & 84.07 & 83.04 & 85.98 \\
toothbrush  & 41.19 & 52.81 & 53.72 & 76.55 & 37.46 & 40.91 & 33.85 & 43.03 & 36.78 & 38.94 & 35.39 & 48.93 & \textbf{67.61} & \textbf{82.34} \\
transistor  & 58.24 & 68.80 & 65.57 & 80.31 & 62.64 & 69.32 & 62.57 & 72.45 & 78.38 & 87.23 & 71.96 & 83.28 & \textbf{87.26} & \textbf{92.75} \\
wood        & 31.75 & 43.27 & 55.25 & 60.82 & 36.31 & 45.67 & 37.23 & 43.37 & 26.36 & 33.13 & 48.90 & 62.57 & \textbf{71.85} & \textbf{82.65} \\
zipper      & 47.51 & 59.24 & \textbf{61.03} & \textbf{68.53} & 59.07 & 69.39 & 60.04 & 71.52 & 44.42 & 51.83 & 56.77 & 66.66 & \textbf{71.28} & \textbf{80.93} \\ 
\hline
\rowcolor{blue!10}
Average     & 45.41 & 55.90 & 57.51 & 66.18 & 51.02 & 59.15 & 49.27 & 56.24 & 52.76 & 58.81 & 52.76 & 63.50 & \textbf{69.36} & \textbf{77.15} \\ 
\hline
\end{tabular}
}
\label{mvtec_stdc}
\end{table*}

\begin{table*}[htbp]
  \centering
  \caption{Comparison of pixel-level anomaly segmentation (mIoU / Acc) on extended BTAD across three real-time backbones.}
  \label{tab:btad_all}
  \resizebox{1\textwidth}{!}{
  \begin{tabular}{c|c|cc|cc|cc|cc|cc|cc|cc}
    \hline
    \rowcolor{blue!10}\textbf{Backbone} & \textbf{Category} 
    & \multicolumn{2}{c|}{\textbf{CutPaste}} 
    & \multicolumn{2}{c|}{\textbf{DRAEM}} 
    & \multicolumn{2}{c|}{\textbf{GLASS}} 
    & \multicolumn{2}{c|}{\textbf{RealNet}} 
    & \multicolumn{2}{c|}{\textbf{DFMGAN}} 
    & \multicolumn{2}{c|}{\textbf{AnomalyDiffusion}} 
    & \multicolumn{2}{c}{\textbf{STAGE}} \\
    \hline
    \rowcolor{blue!10} &  & mIoU $\uparrow$ & Acc $\uparrow$ 
     & mIoU $\uparrow$ & Acc $\uparrow$ 
     & mIoU $\uparrow$ & Acc $\uparrow$ 
     & mIoU $\uparrow$ & Acc $\uparrow$ 
     & mIoU $\uparrow$ & Acc $\uparrow$ 
     & mIoU $\uparrow$ & Acc $\uparrow$ 
     & mIoU $\uparrow$ & Acc $\uparrow$ \\
    \hline
    \multirow{3}{*}{SegFormer} 
    & 01 & 66.94 & 78.20 & 67.86 & 80.14 & 68.02 & 79.57 & 67.17 & 80.20 & 67.02 & 78.03 & 66.55 & 76.31 & \textbf{73.12} & \textbf{84.44} \\
    & 02 & 65.04 & 83.64 & 69.52 & 82.96 & 69.99 & 83.58 & 70.64 & 83.90 & 68.75 & 84.92 & 68.06 & 84.74 & \textbf{70.75} & \textbf{83.44} \\
    & 03 & 50.96 & 60.41 & 50.39 & 54.30 & 51.77 & 53.53 & 48.76 & 57.50 & 38.95 & 41.55 & 54.85 & 80.20 & \textbf{71.38} & \textbf{86.41} \\
    \cdashline{1-16}
    \multirow{3}{*}{BiseNet V2} 
    & 01 & 57.15 & 69.88 & 49.16 & 63.48 & 44.09 & 50.57 & 45.45 & 57.65 & 49.49 & 59.20 & 46.66 & 55.18 & \textbf{63.24} & \textbf{78.59} \\
    & 02 & 59.45 & 82.05 & 66.46 & 80.29 & 66.37 & 79.46 & 66.11 & 81.67 & 66.02 & 79.21 & 65.57 & \textbf{84.00} & \textbf{66.50} & 79.35 \\
    & 03 & 31.84 & 40.62 & 36.15 & 39.04 & 30.80 & 37.15 & 29.55 & 33.11 & 20.12 & 21.48 & 42.27 & 74.41 & \textbf{65.59} & \textbf{85.64} \\
    \cdashline{1-16}
    \multirow{3}{*}{STDC} 
     & 01 & 48.06 & 59.86 & 42.17 & 65.36 & 45.51 & 60.12 & 32.91 & 49.21 & 44.68 & 51.71 & 44.85 & 55.29 & \textbf{64.07} & \textbf{75.94} \\
    & 02 & 59.80 & 77.57 & 64.96 & \textbf{84.32} & 65.02 & 81.94 & 64.00 & 82.64 & 64.85 & 75.32 & 64.73 & 78.93 & \textbf{66.59} & 80.08 \\
    & 03 & 19.76 & 25.20 & 36.14 & 38.80 & 17.04 & 28.01 & 22.57 & 24.79 & 14.67 & 16.55 & 41.71 & 65.45 & \textbf{62.98} & \textbf{75.45} \\
    \hline
  \end{tabular}
  \label{seg_btad}
  }
\end{table*}

\begin{table*}[htbp]
\centering
\caption{Comparison of AUROC and PRO for anomaly synthesis across various methods on extended MVTec AD.}
\resizebox{1\textwidth}{!}{
\begin{tabular}{l|cc|cc|cc|cc|cc|cc|cc}
\hline
\rowcolor{blue!10} \textbf{Category} 
& \multicolumn{2}{c|}{\textbf{CutPaste}} 
& \multicolumn{2}{c|}{\textbf{DRAEM}} 
& \multicolumn{2}{c|}{\textbf{GLASS}} 
& \multicolumn{2}{c|}{\textbf{DFMGAN}} 
& \multicolumn{2}{c|}{\textbf{RealNet}} 
& \multicolumn{2}{c|}{\textbf{AnomalyDiffusion}} 
& \multicolumn{2}{c}{\textbf{STAGE}} \\ \hline
\rowcolor{blue!10} & \textbf{AUROC} $\uparrow$ & \textbf{PRO} $\uparrow$ 
& \textbf{AUROC} $\uparrow$ & \textbf{PRO} $\uparrow$ 
& \textbf{AUROC} $\uparrow$ & \textbf{PRO} $\uparrow$ 
& \textbf{AUROC} $\uparrow$ & \textbf{PRO} $\uparrow$ 
& \textbf{AUROC} $\uparrow$ & \textbf{PRO} $\uparrow$ 
& \textbf{AUROC} $\uparrow$ & \textbf{PRO} $\uparrow$
& \textbf{AUROC} $\uparrow$ & \textbf{PRO} $\uparrow$ \\ \hline
bottle & 99.05 & 82.79 & 97.96 & 86.55 & 97.40 & 70.20 & 99.32 & 81.10 & 98.97 & 85.00 & {99.56} & 83.53 & \textbf{99.89} & \textbf{91.58} \\
cable & 95.29 & 55.04 & 95.95 & 69.30 & 95.07 & 50.76 & 95.57 & 65.15 & 95.36 & 62.79 & 95.74 & 68.66 & \textbf{98.58} & \textbf{84.15} \\
capsule & 98.11 & 44.58 & 99.56 & \textbf{68.17} & 96.50 & 33.59 & 99.26 & 56.40 & 99.19 & 66.21 & 99.16 & 39.59 & \textbf{99.59} & 51.98 \\
carpet & 97.14 & 64.62 & \textbf{99.82} & \textbf{73.79} & 96.42 & 65.94 & 99.36 & 65.13 &96.34 & 65.99 & 98.30 & 63.67 & 99.71 & 73.08 \\
grid & 98.94 & 43.52 & \textbf{99.66} & \textbf{56.60} & 95.64 & 37.91 & 99.20 & 44.49 & 99.60 & 53.17 & 99.33 & 46.25 & 99.54 & 49.05 \\
hazel\_nut & 99.14 & 79.36 & 99.45 & 84.43 & 95.63 & 64.92 & 99.32 & 82.91 & 98.89 & 78.75& {99.76} & 83.48& \textbf{99.79} & \textbf{88.94} \\
leather & 99.85 & 50.32 & 99.90 & 63.34 & 99.78 & 67.41 & 99.88 & 61.91 & \textbf{99.93} & 68.63 & 99.85 & 68.07 & 99.92 & \textbf{70.63} \\
metal\_nut & 99.32 & 69.64 & \textbf{99.54} & 87.03 & 94.80 & 59.76 & 99.50 & 83.67 & 99.46 & 80.54 & 98.80 & 73.52 & 98.90 & \textbf{91.54} \\
pill & 95.38 & 54.40 & 96.70 & 72.77 & 98.29 & 36.60 & \textbf{99.62} & 69.02 & 99.19 & 70.11 & 99.47 & 60.96 & 98.89 & \textbf{76.24} \\
screw & 92.78 & 20.63 & \textbf{99.47} & 49.35 & 94.88 & 23.06 & 99.14 & \textbf{51.86} & 98.99 & 48.42 & 96.31 & 36.11 & 97.63 & 38.76 \\
tile & 98.60 & 82.83 & 99.73 & 87.80 & 99.07 & 83.37 & \textbf{99.74} & \textbf{88.59} & 99.31 & 82.91 & 99.53 & 83.14 & 98.87 & 87.63 \\
toothbrush & 88.21 & 29.07 & 98.53 & 69.12 & 85.42 & 30.15 & 97.95 & 62.25 & 96.37 & 60.77 & 96.23 & 37.67 & \textbf{99.33} & \textbf{74.54} \\
transistor & 96.40 & 55.03 & 92.80 & 68.10 & 94.41 & 49.77 & 97.17 & 70.57 & 95.59 & 68.92 & 98.84 & 72.85 & \textbf{98.88} & \textbf{82.20} \\
wood & 94.01 & 64.46 & 99.24 & \textbf{78.80} & 95.94 & 54.97 & \textbf{99.31} & 72.97 & 98.76 & 72.49 & 97.52 & 66.84 & 97.62 & 78.00 \\
zipper & 99.58 & 74.62 & 99.68 & 78.02 & 99.72 & 77.09 & 99.47 & 73.73 & 99.69 & 77.11 & 99.65 & 74.82 & \textbf{99.86} & \textbf{80.43} \\
\rowcolor{blue!10} Average & 96.79 & 58.06 & 98.53 & 72.88 & 95.93 & 53.70 & 98.92 & 68.65 & 98.38 & 69.45 & 98.54 & 63.94 & \textbf{99.34} & \textbf{74.86} \\
\hline    
\end{tabular}
}
\label{pro}
\end{table*}

\begin{table*}[!t]
\centering
\caption{Comparison of AP and F1 for anomaly synthesis across various methods.}
\resizebox{1\textwidth}{!}{
\begin{tabular}{l|cc|cc|cc|cc|cc|cc|cc}
\hline
\rowcolor{blue!10} \textbf{Category} 
& \multicolumn{2}{c|}{\textbf{CutPaste}} 
& \multicolumn{2}{c|}{\textbf{DRAEM}} 
& \multicolumn{2}{c|}{\textbf{GLASS}} 
& \multicolumn{2}{c|}{\textbf{DFMGAN}} 
& \multicolumn{2}{c|}{\textbf{RealNet}} 
& \multicolumn{2}{c|}{\textbf{AnomalyDiffusion}} 
& \multicolumn{2}{c}{\textbf{STAGE}} \\ \hline
\rowcolor{blue!10} & \textbf{F1} $\uparrow$ & \textbf{AP} $\uparrow$ 
& \textbf{F1} $\uparrow$ & \textbf{AP} $\uparrow$ 
& \textbf{F1} $\uparrow$ & \textbf{AP} $\uparrow$ 
& \textbf{F1} $\uparrow$ & \textbf{AP} $\uparrow$ 
& \textbf{F1} $\uparrow$ & \textbf{AP} $\uparrow$ 
& \textbf{F1} $\uparrow$ & \textbf{AP} $\uparrow$
& \textbf{F1} $\uparrow$ & \textbf{AP} $\uparrow$ \\ \hline
bottle & 78.56 & 92.72 & 83.59 & 94.71 & 65.36 & 90.69 & 78.10 & 92.63 & 82.12 & 94.80 & 81.66 & 93.32 & \textbf{91.11} & \textbf{97.20} \\
cable & 42.08 & 74.09 & 60.05 & 79.94 & 39.08 & 78.01 & 53.97 & 80.27 & 54.02 & 75.98 & 62.75 & 76.15 & \textbf{82.80} & \textbf{91.52} \\
capsule & 33.04 & 64.55 & \textbf{58.21} & \textbf{83.16} & 24.28 & 70.58 & 43.32 & 76.68 & 54.93 & 80.46 & 25.36 & 67.95 & 41.78 & 77.97 \\
carpet & 58.55 & 75.72 & \textbf{73.47} & \textbf{90.52} & 61.51 & 82.11 & 55.83 & 85.07 & 61.01 & 77.90 & 56.94 & 78.20 & 67.24 & 84.55 \\
grid & 29.60 & 56.69 & \textbf{37.27} & \textbf{74.88} & 18.67 & 65.93 & 18.90 & 71.48 & 30.89 & 73.87 & 31.63 & 61.12 & 30.97 & 68.25 \\
hazel\_nut & 73.84 & 91.93 &82.07 & 94.21 & 55.25 & 86.26 & 79.31 & 93.91 & 72.50 & 92.36 & 79.81 & 93.74 &\textbf{87.84} & \textbf{96.76} \\
leather & 36.79 & 84.05 & 51.78 & \textbf{90.80} & 57.11 & 87.92 & 50.40 & 85.78 & 56.96 & 90.03 & 55.11 & 85.66 & \textbf{61.77} & 88.52 \\
metal\_nut & 57.91 & 88.55 & 85.31 & 95.21 & 53.24 & 81.01 & 78.46 & 95.86 & 75.39 & 91.88 & 67.02 & 84.73 & \textbf{91.22} & \textbf{97.26} \\
pill & 42.99 & 69.69 & 67.05 & 86.85 & 22.91 & 77.72 & 60.39 & {87.41} & 60.54 & 85.21 & 52.58 & 77.89 & \textbf{74.52} & \textbf{88.08} \\
screw & 15.66 & 43.08 & 41.62 & 68.24 & 18.21 & 61.34 & \textbf{43.27} & \textbf{73.58} & 40.37 & 69.08 & 28.76 & 62.03 & 31.02 & 63.08 \\
tile & 77.73 & 92.94 & 84.46 & \textbf{96.78} & 78.09 & 94.37 & \textbf{85.58} & 96.69 & 76.37 & 93.55 & 76.24 & 93.43 & 84.40 & 95.99 \\
toothbrush & 23.94 & 36.97 & 60.48 & 80.52 & 20.95 & 43.77 & 47.94 & 77.66 & 49.35 & 67.16 & 33.76 & 53.65 & \textbf{71.25} & \textbf{82.67} \\
transistor & 42.97 & 73.01 & 59.39 & 77.54 & 18.29 & 73.25 & 55.97 & 82.68 & 61.32 & 78.10 & 68.60 & 82.04 & \textbf{81.93} & \textbf{93.80} \\
wood & 59.02 & 73.64 & 73.16 & \textbf{91.12} & 46.06 & 77.78 & 64.42 & 89.61 & 65.24 & 87.16 & 61.79 & 77.53 & \textbf{75.66} & 85.84 \\
zipper & 66.88 & 90.10 & 72.17 & 92.48 & 70.51 & 92.95 & 66.40 & 89.54 & 71.53 & 90.93 & 68.55 & 90.96 & \textbf{79.02} & \textbf{94.88} \\
\rowcolor{blue!10} Average & 49.30 & 73.85 & 66.01 & 86.48 & 43.30 & 77.58 & 58.84 & 85.24 & 60.84 & 83.23 & 56.70 & 78.55 & \textbf{70.07} & \textbf{87.02} \\
\hline    
\end{tabular}
}
\label{tab:ap_f1}
\end{table*}

\subsection{Main Results}

\subsubsection{Comparisons on Anomaly Segmentation}
The quantitative results on anomaly segmentation indicate that STAGE consistently outperforms other SIAS methods. To demonstrate this, we synthesize 500 image–mask pairs for each anomaly type across various categories in both the MVTec~AD and BTAD datasets. We then merge these synthetic pairs with about one-third of the real image–mask pairs for training, and use the remaining real pairs for evaluation. As shown in Table~\ref{Segformer_mvtec}(MVTec~AD), Table~\ref{mvtec_bise}(MVTec~AD), Table~\ref{mvtec_stdc}(MVTec~AD) and Table~\ref{seg_btad}(BTAD), these models trained with STAGE’s synthetic data exhibit marked improvements in almost all categories on three segmentation backbones. 
\begin{figure}[H]
  \centering
  \includegraphics[width=\linewidth,trim=0.6mm 0.4mm 0.6mm 0mm,clip]{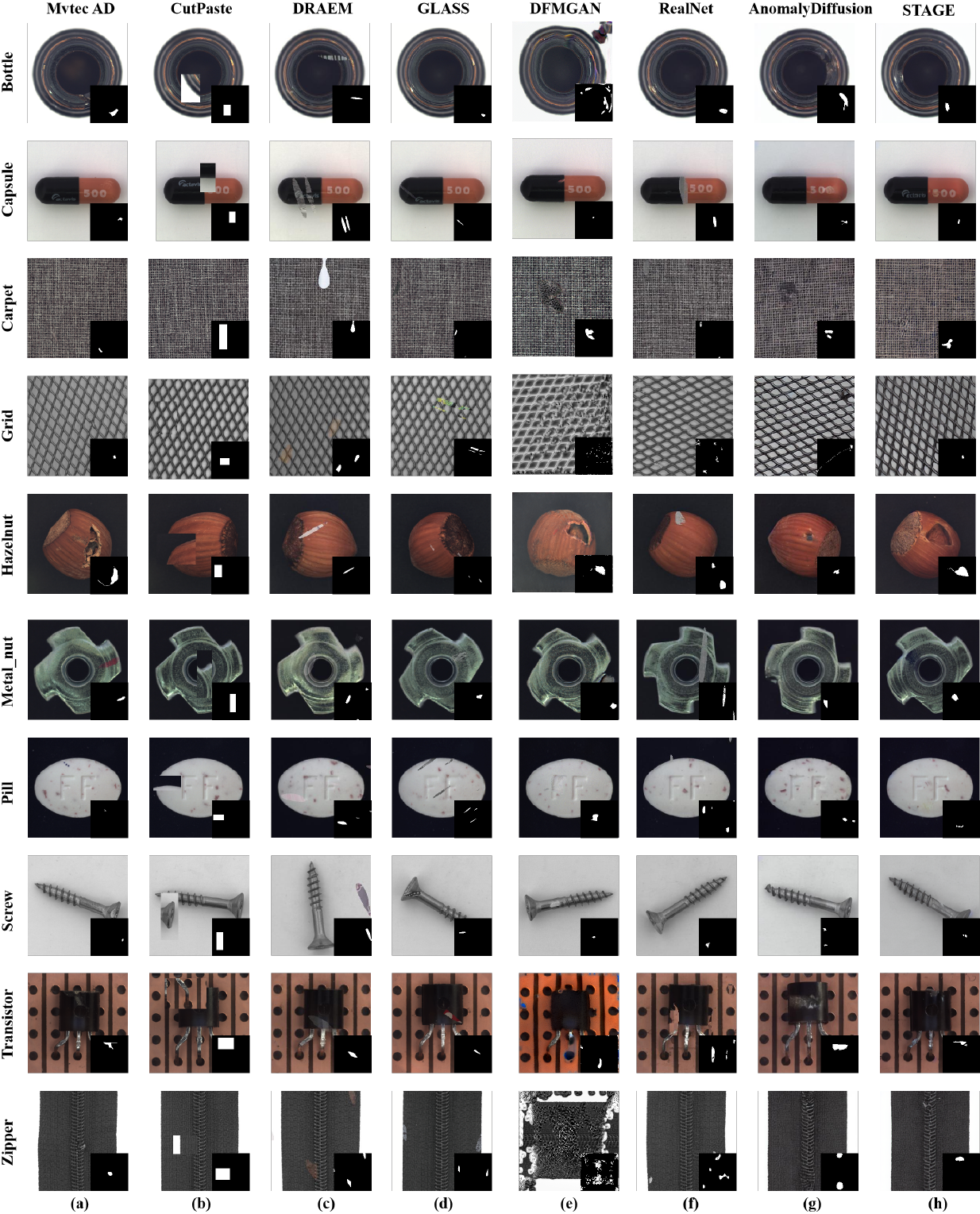} 
  \caption{Comparison of anomaly synthesis methods on \textbf{MVTec AD}. Each column represents a method (Ground Truth (MVTec AD), CutPaste, DRAEM, GLASS, DFMGAN, RealNet, AnomalyDiffusion, and our \textbf{STAGE}). Each row shows an object with its mask annotations.}
  \label{fig:MV}
\end{figure}

For capsules with tiny abnormal regions, STAGE improves the segmentation mIoU by 8.57\%, 13.04\%, and 10.34\% over the second-best methods when using SegFormer, BiseNet V2, and STDC, respectively. On average, STAGE improves the mIoU by 7.74\% over the second-best method DFMGAN when using the SegFormer backbone. These results underscore the effectiveness of STAGE in enhancing SIAS, especially for cases with tiny anomalies.

\begin{figure*}[htb]
    \centering
\includegraphics[width=1\linewidth]{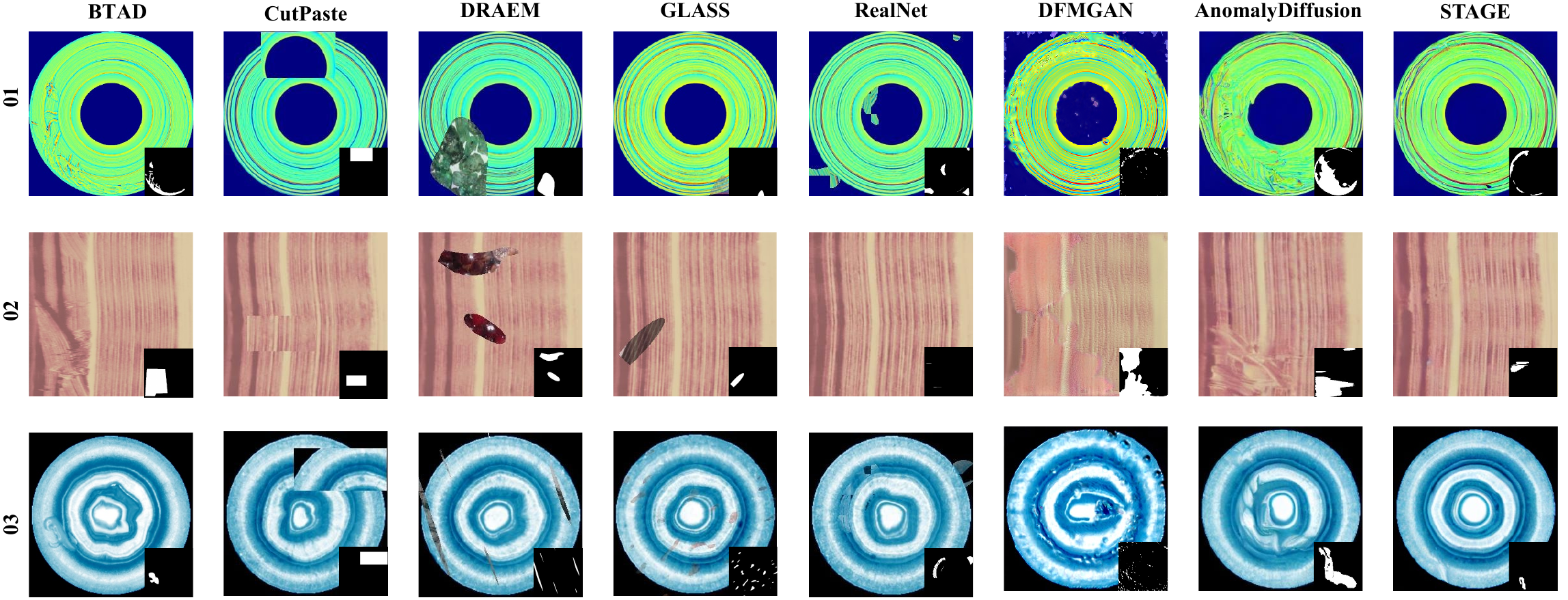}
    \caption{Comparison of different anomaly synthesis methods. Each column represents a different method, including Ground Truth (BTAD), CutPaste, DRAEM, RealNet, DFMGAN, AnomalyDiffusion, and our \textbf{STAGE}. Each row corresponds to a different object with its mask annotations.}
    \label{BTAD}
\end{figure*}
\subsubsection{Comparisons on Anomaly Detection for Auxiliary Analysis}

To provide a more comprehensive evaluation of the performance of different methods on SIAS, we additionally report their results on conventional metrics. Specifically, anomaly maps are extracted from the SegFormer logits by selecting the channel corresponding to anomalies, and standard metrics are computed based on these maps. This strategy avoids coupling the evaluation with model-specific downstream architectures or training pipelines, thereby preventing biases unrelated to intrinsic synthesis quality. The adoption of a unified backbone ensures that the evaluation reflects only the quality of the synthesized anomalies rather than differences in downstream detection models. As shown in Table~\ref{pro}, STAGE achieves advanced performance in terms of AUROC and PRO across most categories. Moreover, Table~\ref{tab:ap_f1} further demonstrates that STAGE delivers competitive performance in terms of AP and F1-score, underscoring the robustness of its anomaly synthesis quality.

\subsubsection{Comparisons on Anomaly Synthesis}
Furthermore, Fig.~\ref{fig:MV} and Fig.~\ref{BTAD} illustrate the high-quality anomalies generated by STAGE on the MVTec and BTAD datasets. Specifically, these results show that STAGE generates realistic anomalies that precisely align with the background. In contrast, other existing models exhibit various limitations: (\romannumeral1) CutPaste, DRAEM, and GLASS rely on data augmentation or patch-based sampling from other datasets, resulting in anomalies with unnatural textures and poor integration with the background. (\romannumeral2) RealNet perturbs the variance of the denoising distribution to synthesize anomalies, but these anomalies remain scattered. As shown in the transistor in Fig.~\ref{fig:MV}, although some structural consistency is preserved, the anomalies fail to seamlessly blend with the background. (\romannumeral3) DFMGAN and AnomalyDiffusion aim to model the underlying anomaly distribution but fail to properly match anomalies to the mask annotations. For instance, DFMGAN produces a fractured background for the zipper class, and abnormal areas produced by AnomalyDiffusion are smaller than the assigned mask for the pill class. In contrast, STAGE consistently generates anomalies that are well-aligned with the mask annotations and exhibit realistic texture details, thus enabling more accurate and reliable anomaly synthesis for downstream segmentation.

\begin{figure*}[hbt]
  \centering
  \includegraphics[width=\linewidth]{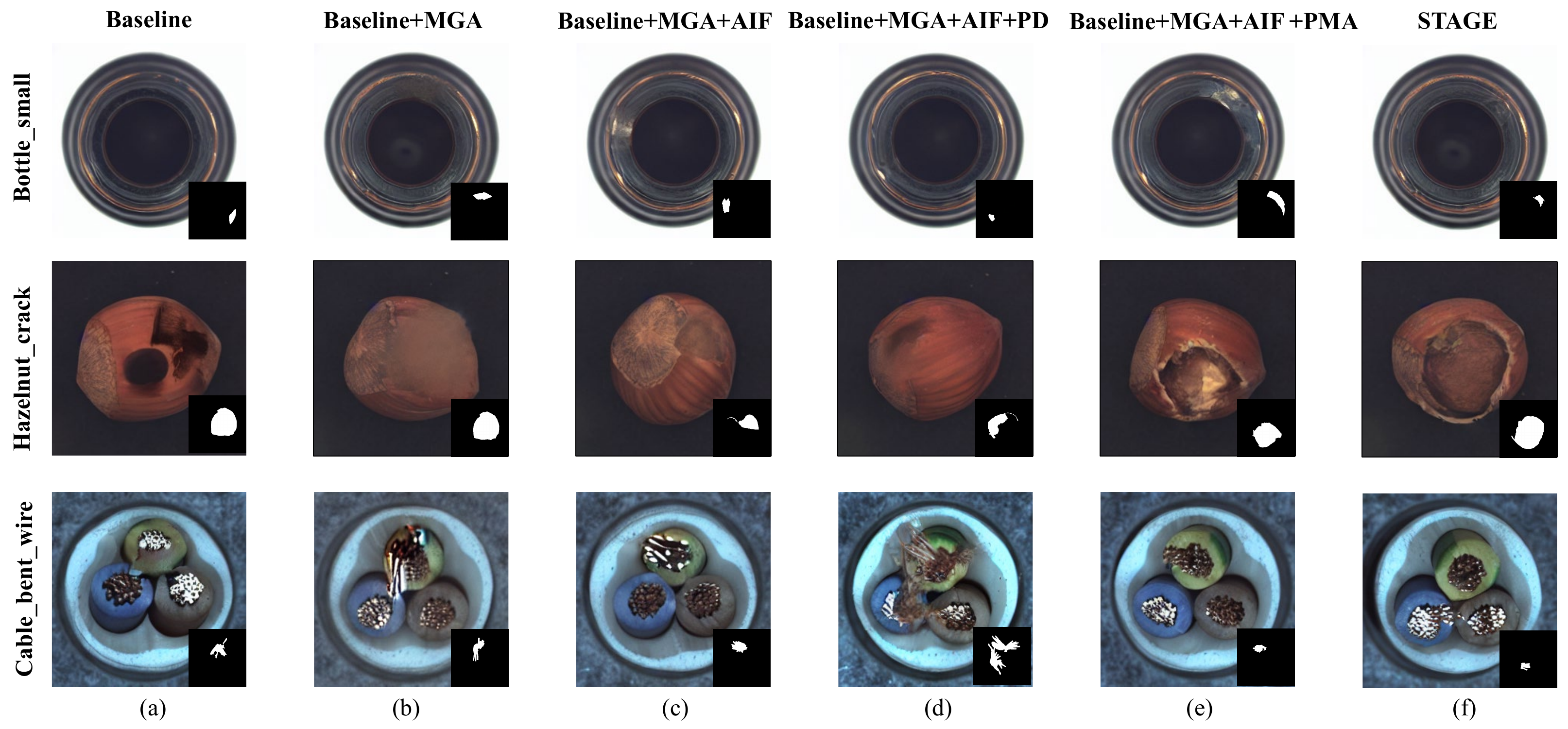}
  \caption{
    Qualitative comparison of ablation studies. Each row corresponds to the different object and anomaly type, while each column represents results from the model with different modules.
  }
  \label{net5}
\end{figure*}

\subsection{Ablation Study on STAGE}
\begin{figure*}[!t]
  \centering
  \includegraphics[width=1\linewidth]{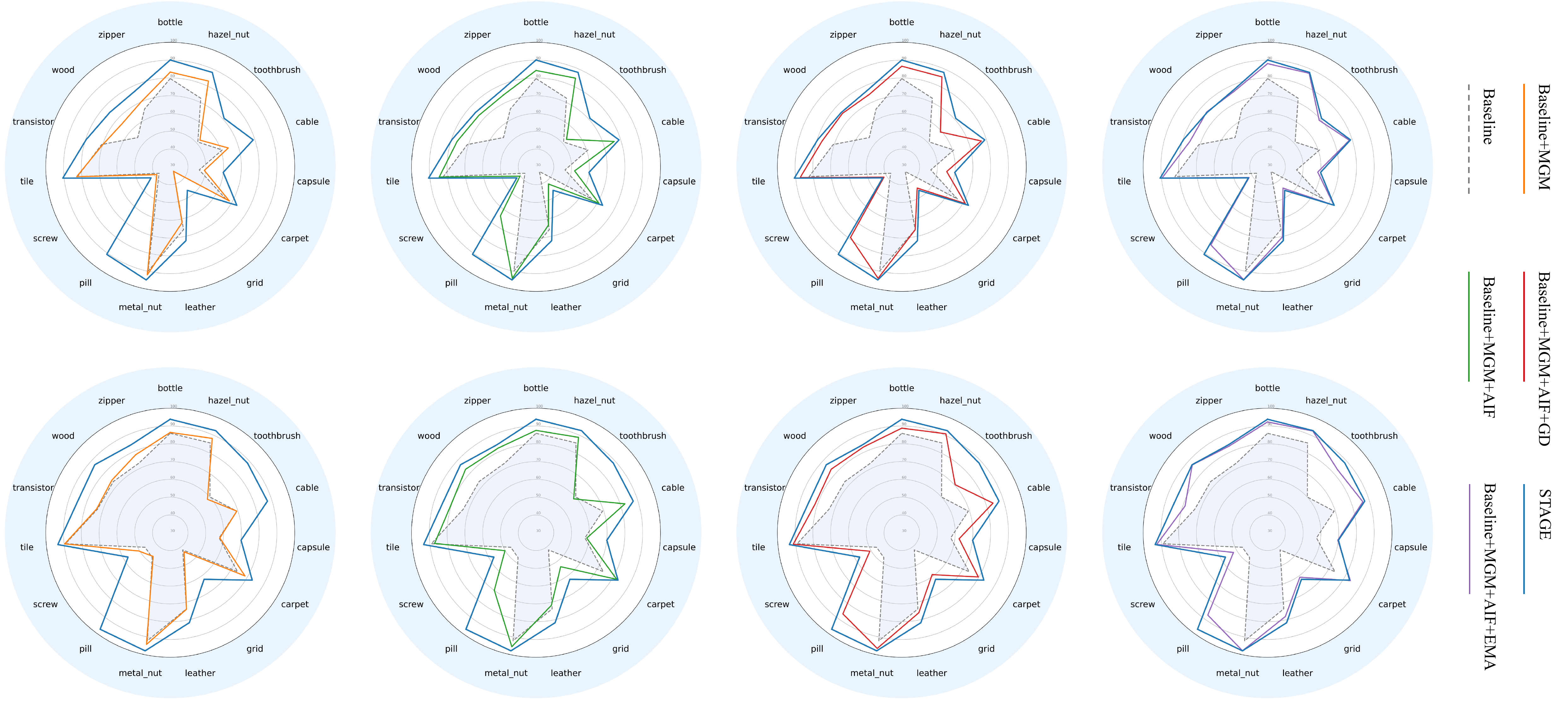}
  \caption{The effect of different modules of STAGE in the MVTec dataset. \textbf{Top row} illustrate per-category segmentation results (mIoU), \textbf{bottom row} present classification accuracy (Acc).}
  \label{radar:ablation}
\end{figure*}

We conduct ablation studies on all categories in the MVTec dataset to evaluate Anomaly Inference (AIF), Graded Diffusion (GD), Explicit Mask Alignment (EMA), and Mask Guidance Adapter (MGA). Six configurations are employed: (1) baseline, (2) baseline + MGA, (3) baseline + MGA + AIF, (4) baseline + MGA + AIF + GD, (5) baseline + MGA + AIF + EMA, and (6) baseline + MGA + AIF + GD + EMA (i.e., the proposed STAGE). Here, the baseline refers to the diffusion model without any additional modules. For each setup, we generate 500 abnormal image–mask pairs per anomaly type to train a SegFormer for anomaly segmentation. Fig.~\ref{radar:ablation} and Fig.~\ref{net5} present the quantitative and qualitative results, respectively.

\textbf{Effectiveness of MGA.}  
Fig.~\ref{radar:ablation} illustrates that equipping the model with MGA leads to consistent improvements across nearly all MVTec categories. Specifically, the average mIoU and accuracy increase from 58.19\% and 67.30\% (Baseline) to 65.45\% and 73.27\%. Notable per-category gains are observed on challenging classes such as {hazel\_nut} (+10.56\% mIoU) and {wood} (+11.74\% mIoU). These results confirm that MGA effectively adapts the pre-trained LDM to the anomaly domain, providing targeted guidance without updating the backbone, serving as a lightweight and efficient adaptation module.

\textbf{Effectiveness of AIF.}  
Results demonstrate that AIF outperforms conventional inpainting-based synthesis. This improvement comes from explicitly integrating clean background information into the denoising distribution, which suppresses redundant background reconstruction and directs the model to generate anomalies strictly within masked regions. As shown in Fig.~\ref{net5}(b), the synthesized anomalies in the {hazel\_nut} exhibit clearer boundaries and stronger contrast, thereby providing more reliable guidance for downstream segmentation.

\textbf{Effectiveness of GD.}  
GD prevents small or low-contrast anomalies from being overlooked. As shown in Fig.~\ref{net5}, GD enables accurate generation of anomalies corresponding to small anomaly masks (e.g., the {bottle} in Fig.~\ref{net5}(c)), thereby improving the quality of the synthesized data. Compared to the model with AIF alone, GD yields performance improvements of 2.4\%, 0.8\%, and 6.16\% across the three segmentation models, as shown in Fig.~\ref{radar:ablation}.

\textbf{Effectiveness of EMA.}  
As shown in Fig.~\ref{net5}, integrating EMA into AIF results in smoother transitions between anomalies and background. Unlike binary mask fusion in “AIF+PD”, which introduces boundary artifacts, EMA progressively aligns the mask over time, enabling coherent anomaly–background integration. This refinement improves segmentation: compared to “AIF+GD”, EMA achieves mIoU gains of 1.45\%, 2.17\%, and 9.85\% across different models. Compared to “AIF” alone, the improvements are even larger: 3.85\%, 2.97\%, and 16.01\%. Notably, in the \textit{toothbrush}, which contains highly complex and fragmented masks, the spatial alignment introduced by EMA yields the most substantial improvement.

Overall, omitting AIF, GD, or EMA degrades STAGE’s performance, with the steepest declines when multiple components are removed. Integrating all three achieves the highest mIoU and accuracy, confirming their complementary roles and the overall effectiveness of STAGE.

\subsection{Limitations}

However, it can be observed that the segmentation performance for the \textit{screw} and \textit{grid} categories consistently falls short of the best results across different segmentation backbones. Upon closer inspection, we attribute this limitation to the inherent mismatch between anomaly masks and the objects themselves. Unlike categories such as \textit{zipper} and \textit{wood}, where the objects occupy a substantial portion of the image, the \textit{screw} and \textit{grid} objects cover only a small area. During the synthesis process, the anomaly masks cannot be precisely aligned with the object regions and are often misaligned and projected onto background regions. This misalignment introduces erroneous contextual information to EMA, ultimately leading to suboptimal results. Therefore, a promising direction for future research is to design region-aware mask synthesis strategies, particularly tailored for small objects.

\section{Conclusion}\label{sec:conclusion}

In this work, we introduce STAGE (Segmentation-oriented Anomaly Synthesis via Graded Diffusion with Explicit Mask Alignment), a novel framework for industrial anomaly synthesis. STAGE addresses key limitations of existing methods—unrealistic textures, misalignment between anomalies and the background, and omission of tiny anomalies. It comprises three complementary components: (1) \textit{anomaly inference} that conditions denoising on clean background into denoising to suppress redundant reconstruction and emphasize anomaly generation; (2) a \textit{graded diffusion} strategy with a dual-branch design that preserves small anomaly regions throughout sampling; and (3) \textit{explicit mask alignment} that progressively adjusts spatial blending between anomaly and background to ensure smooth boundaries. Extensive experiments on MVTec and BTAD show that STAGE significantly outperforms state-of-the-art anomaly synthesis methods in segmentation. These results underscore STAGE’s potential to advance high-precision industrial inspection through effective, controllable anomaly synthesis.
\bibliographystyle{IEEEtran}
\bibliography{ref}

\begin{thebibliography}{10}
\providecommand{\url}[1]{#1}
\csname url@samestyle\endcsname
\providecommand{\newblock}{\relax}
\providecommand{\bibinfo}[2]{#2}
\providecommand{\BIBentrySTDinterwordspacing}{\spaceskip=0pt\relax}
\providecommand{\BIBentryALTinterwordstretchfactor}{4}
\providecommand{\BIBentryALTinterwordspacing}{\spaceskip=\fontdimen2\font plus
\BIBentryALTinterwordstretchfactor\fontdimen3\font minus \fontdimen4\font\relax}
\providecommand{\BIBforeignlanguage}[2]{{%
\expandafter\ifx\csname l@#1\endcsname\relax
\typeout{** WARNING: IEEEtran.bst: No hyphenation pattern has been}%
\typeout{** loaded for the language `#1'. Using the pattern for}%
\typeout{** the default language instead.}%
\else
\language=\csname l@#1\endcsname
\fi
#2}}
\providecommand{\BIBdecl}{\relax}
\BIBdecl

\bibitem{Baitieva_2024_CVPR}
A.~Baitieva, D.~Hurych, V.~Besnier, and O.~Bernard, ``Supervised anomaly detection for complex industrial images,'' in \emph{Proceedings of the IEEE/CVF Conference on Computer Vision and Pattern Recognition (CVPR)}, June 2024, pp. 17\,754--17\,762.

\bibitem{pmlr-v202-chu23b}
\BIBentryALTinterwordspacing
Y.-M. Chu, C.~Liu, T.-I. Hsieh, H.-T. Chen, and T.-L. Liu, ``Shape-guided dual-memory learning for 3{D} anomaly detection,'' in \emph{Proceedings of the 40th International Conference on Machine Learning}, ser. Proceedings of Machine Learning Research, A.~Krause, E.~Brunskill, K.~Cho, B.~Engelhardt, S.~Sabato, and J.~Scarlett, Eds., vol. 202.\hskip 1em plus 0.5em minus 0.4em\relax PMLR, 23--29 Jul 2023, pp. 6185--6194. [Online]. Available: \url{https://proceedings.mlr.press/v202/chu23b.html}
\BIBentrySTDinterwordspacing

\bibitem{qiu2022latent}
C.~Qiu, A.~Li, M.~Kloft, M.~Rudolph, and S.~Mandt, ``Latent outlier exposure for anomaly detection with contaminated data,'' in \emph{International conference on machine learning}.\hskip 1em plus 0.5em minus 0.4em\relax PMLR, 2022, pp. 18\,153--18\,167.

\bibitem{li2021cutpaste}
C.-L. Li, K.~Sohn, J.~Yoon, and T.~Pfister, ``Cutpaste: Self-supervised learning for anomaly detection and localization,'' in \emph{Proceedings of the IEEE/CVF conference on computer vision and pattern recognition}, 2021, pp. 9664--9674.

\bibitem{zavrtanik2021draem}
V.~Zavrtanik, M.~Kristan, and D.~Sko{\v{c}}aj, ``Draem-a discriminatively trained reconstruction embedding for surface anomaly detection,'' in \emph{Proceedings of the IEEE/CVF international conference on computer vision}, 2021, pp. 8330--8339.

\bibitem{dfmgan}
Y.~Duan, Y.~Hong, L.~Niu, and L.~Zhang, ``Few-shot defect image generation via defect-aware feature manipulation,'' in \emph{Proceedings of the AAAI Conference on Artificial Intelligence}, vol.~37, no.~1, 2023, pp. 571--578.

\bibitem{hu2024anomalydiffusion}
T.~Hu, J.~Zhang, R.~Yi, Y.~Du, X.~Chen, L.~Liu, Y.~Wang, and C.~Wang, ``Anomalydiffusion: Few-shot anomaly image generation with diffusion model,'' in \emph{Proceedings of the AAAI Conference on Artificial Intelligence}, vol.~38, no.~8, 2024, pp. 8526--8534.

\bibitem{niu2020defect}
S.~Niu, B.~Li, X.~Wang, and H.~Lin, ``Defect image sample generation with gan for improving defect recognition,'' \emph{IEEE Transactions on Automation Science and Engineering}, vol.~17, no.~3, pp. 1611--1622, 2020.

\bibitem{liu2019multistage}
J.~Liu, C.~Wang, H.~Su, B.~Du, and D.~Tao, ``Multistage gan for fabric defect detection,'' \emph{IEEE Transactions on Image Processing}, vol.~29, pp. 3388--3400, 2019.

\bibitem{lugmayr2022repaint}
A.~Lugmayr, M.~Danelljan, A.~Romero, F.~Yu, R.~Timofte, and L.~Van~Gool, ``Repaint: Inpainting using denoising diffusion probabilistic models,'' in \emph{Proceedings of the IEEE/CVF conference on computer vision and pattern recognition}, 2022, pp. 11\,461--11\,471.

\bibitem{jin2024dualanodiff}
Y.~Jin, J.~Peng, Q.~He, T.~Hu, H.~Chen, J.~Wu, W.~Zhu, M.~Chi, J.~Liu, Y.~Wang \emph{et~al.}, ``Dualanodiff: Dual-interrelated diffusion model for few-shot anomaly image generation,'' \emph{arXiv preprint arXiv:2408.13509}.

\bibitem{yang2025defect}
S.~Yang, Z.~Chen, P.~Chen, X.~Fang, Y.~Liang, S.~Liu, and Y.~Chen, ``Defect spectrum: a granular look of large-scale defect datasets with rich semantics,'' in \emph{European Conference on Computer Vision}.\hskip 1em plus 0.5em minus 0.4em\relax Springer, 2025, pp. 187--203.

\bibitem{lyu2024reb}
S.~Lyu, D.~Mo, and W.~keung Wong, ``Reb: Reducing biases in representation for industrial anomaly detection,'' \emph{Knowledge-Based Systems}, vol. 290, p. 111563, 2024.

\bibitem{zhang2023destseg}
X.~Zhang, S.~Li, X.~Li, P.~Huang, J.~Shan, and T.~Chen, ``Destseg: Segmentation guided denoising student-teacher for anomaly detection,'' in \emph{Proceedings of the IEEE/CVF Conference on Computer Vision and Pattern Recognition}, 2023, pp. 3914--3923.

\bibitem{zhang2021defect}
G.~Zhang, K.~Cui, T.-Y. Hung, and S.~Lu, ``Defect-gan: High-fidelity defect synthesis for automated defect inspection,'' in \emph{Proceedings of the IEEE/CVF Winter Conference on Applications of Computer Vision}, 2021, pp. 2524--2534.

\bibitem{du2022new}
Z.~Du, L.~Gao, and X.~Li, ``A new contrastive gan with data augmentation for surface defect recognition under limited data,'' \emph{IEEE Transactions on Instrumentation and Measurement}, vol.~72, pp. 1--13, 2022.

\bibitem{dai2024seas}
Z.~Dai, S.~Zeng, H.~Liu, X.~Li, F.~Xue, and Y.~Zhou, ``Seas: Few-shot industrial anomaly image generation with separation and sharing fine-tuning,'' \emph{arXiv preprint arXiv:2410.14987}, 2024.

\bibitem{sun2024cut}
H.~Sun, Y.~Cao, and O.~Fink, ``Cut: A controllable, universal, and training-free visual anomaly generation framework,'' \emph{arXiv preprint arXiv:2406.01078}, 2024.

\bibitem{eccv2025}
Q.~Shi, J.~Wei, F.~Shen, and Z.~Zhang, ``Few-shot defect image generation based on consistency modeling,'' in \emph{European Conference on Computer Vision}.\hskip 1em plus 0.5em minus 0.4em\relax Springer, 2025, pp. 360--376.

\bibitem{qu2025dictas}
Z.~Qu, X.~Tao, X.~Gong, S.~Qu, X.~Zhang, X.~Wang, F.~Shen, Z.~Zhang, M.~Prasad, and G.~Ding, ``Dictas: A framework for class-generalizable few-shot anomaly segmentation via dictionary lookup,'' \emph{arXiv preprint arXiv:2508.13560}, 2025.

\bibitem{xu2023fascinating}
H.~Xu, Y.~Wang, J.~Wei, S.~Jian, Y.~Li, and N.~Liu, ``Fascinating supervisory signals and where to find them: Deep anomaly detection with scale learning,'' in \emph{International Conference on Machine Learning}.\hskip 1em plus 0.5em minus 0.4em\relax PMLR, 2023, pp. 38\,655--38\,673.

\bibitem{recon1}
V.~Zavrtanik, M.~Kristan, and D.~Sko{\v{c}}aj, ``Reconstruction by inpainting for visual anomaly detection,'' \emph{Pattern Recognition}, vol. 112, p. 107706, 2021.

\bibitem{recon2}
Y.~Zhou, ``Rethinking reconstruction autoencoder-based out-of-distribution detection,'' in \emph{Proceedings of the IEEE/CVF Conference on Computer Vision and Pattern Recognition}, 2022, pp. 7379--7387.

\bibitem{liu2024dual}
X.~Liu, J.~Wang, B.~Leng, and S.~Zhang, ``Dual-modeling decouple distillation for unsupervised anomaly detection,'' in \emph{Proceedings of the 32nd ACM International Conference on Multimedia}, 2024, pp. 5035--5044.

\bibitem{yang2024slsg}
M.~Yang, J.~Liu, Z.~Yang, and Z.~Wu, ``Slsg: Industrial image anomaly detection with improved feature embeddings and one-class classification,'' \emph{Pattern Recognition}, vol. 156, p. 110862, 2024.

\bibitem{chen2024progressive}
Q.~Chen, H.~Luo, H.~Gao, C.~Lv, and Z.~Zhang, ``Progressive boundary guided anomaly synthesis for industrial anomaly detection,'' \emph{IEEE Transactions on Circuits and Systems for Video Technology}, 2024.

\bibitem{glass}
Q.~Chen, H.~Luo, C.~Lv, and Z.~Zhang, ``A unified anomaly synthesis strategy with gradient ascent for industrial anomaly detection and localization,'' in \emph{European Conference on Computer Vision}.\hskip 1em plus 0.5em minus 0.4em\relax Springer, 2025, pp. 37--54.

\bibitem{mou2024synth4seg}
S.~Mou, R.~Vemulapalli, S.~Li, Y.~Liu, C.~Thomas, M.~Cao, H.~Bai, O.~Tuzel, P.~Huang, J.~Shan \emph{et~al.}, ``Synth4seg--learning defect data synthesis for defect segmentation using bi-level optimization,'' \emph{arXiv preprint arXiv:2410.18490}, 2024.

\bibitem{rombach2022high}
R.~Rombach, A.~Blattmann, D.~Lorenz, P.~Esser, and B.~Ommer, ``High-resolution image synthesis with latent diffusion models,'' in \emph{Proceedings of the IEEE/CVF conference on computer vision and pattern recognition}, 2022, pp. 10\,684--10\,695.

\bibitem{song2021scorebased}
\BIBentryALTinterwordspacing
Y.~Song, J.~Sohl-Dickstein, D.~P. Kingma, A.~Kumar, S.~Ermon, and B.~Poole, ``Score-based generative modeling through stochastic differential equations,'' in \emph{International Conference on Learning Representations}, 2021. [Online]. Available: \url{https://openreview.net/forum?id=PxTIG12RRHS}
\BIBentrySTDinterwordspacing

\bibitem{lee2023convergence}
H.~Lee, J.~Lu, and Y.~Tan, ``Convergence of score-based generative modeling for general data distributions,'' in \emph{International Conference on Algorithmic Learning Theory}.\hskip 1em plus 0.5em minus 0.4em\relax PMLR, 2023, pp. 946--985.

\bibitem{dou2024theory}
Z.~Dou, M.~Chen, M.~Wang, and Z.~Yang, ``Theory of consistency diffusion models: Distribution estimation meets fast sampling,'' in \emph{Forty-first International Conference on Machine Learning}, 2024.

\bibitem{lyu2024sampling}
J.~Lyu, Z.~Chen, and S.~Feng, ``Sampling is as easy as keeping the consistency: convergence guarantee for consistency models,'' in \emph{Forty-first International Conference on Machine Learning}, 2024.

\bibitem{bergmann2019mvtec}
P.~Bergmann, M.~Fauser, D.~Sattlegger, and C.~Steger, ``Mvtec ad--a comprehensive real-world dataset for unsupervised anomaly detection,'' in \emph{Proceedings of the IEEE/CVF conference on computer vision and pattern recognition}, 2019, pp. 9592--9600.

\bibitem{btad}
P.~Mishra, R.~Verk, D.~Fornasier, C.~Piciarelli, and G.~L. Foresti, ``{VT-ADL}: A vision transformer network for image anomaly detection and localization,'' in \emph{30th IEEE/IES International Symposium on Industrial Electronics (ISIE)}, June 2021.

\bibitem{zhang2024realnet}
X.~Zhang, M.~Xu, and X.~Zhou, ``Realnet: A feature selection network with realistic synthetic anomaly for anomaly detection,'' in \emph{Proceedings of the IEEE/CVF Conference on Computer Vision and Pattern Recognition}, 2024, pp. 16\,699--16\,708.

\bibitem{xie2021segformer}
E.~Xie, W.~Wang, Z.~Yu, A.~Anandkumar, J.~M. Alvarez, and P.~Luo, ``Segformer: Simple and efficient design for semantic segmentation with transformers,'' \emph{Advances in neural information processing systems}, vol.~34, pp. 12\,077--12\,090, 2021.

\bibitem{yu2021bisenet}
C.~Yu, C.~Gao, J.~Wang, G.~Yu, C.~Shen, and N.~Sang, ``Bisenet v2: Bilateral network with guided aggregation for real-time semantic segmentation,'' \emph{International journal of computer vision}, vol. 129, pp. 3051--3068, 2021.

\bibitem{stdc}
M.~Fan, S.~Lai, J.~Huang, X.~Wei, Z.~Chai, J.~Luo, and X.~Wei, ``Rethinking bisenet for real-time semantic segmentation,'' in \emph{Proceedings of the IEEE/CVF conference on computer vision and pattern recognition}, 2021, pp. 9716--9725.

\end{thebibliography}

\end{document}